\documentclass[accepted]{uai2026} % after acceptance, for a revised version; 
\usepackage[american]{babel}
\usepackage{amsmath}
\usepackage{amsfonts}
\usepackage{natbib} % has a nice set of citation styles and commands
\usepackage{mathtools} % amsmath with fixes and additions
\usepackage{booktabs} % commands to create good-looking tables
\usepackage{tikz} % nice language for creating drawings and diagrams
\usepackage{mathtools}

\usepackage{amsthm}
\usepackage{amsmath}
\usepackage{amssymb}
\usepackage{float}
\usepackage{caption}
\newtheorem{assumption}{Assumption}
\newtheorem{definition}{Definition}
\newtheorem{proposition}{Proposition}
\newtheorem*{proposition*}{Proposition}
\newtheorem{theorem}{Theorem}
\newtheorem*{theorem*}{Theorem}
\newtheorem{lemma}{Lemma}
\newtheorem{remark}{Remark}
\newtheorem{corollary}{Corollary}

\newcommand{\indep}{\perp\!\!\!\!\perp} 
\newcommand{\notindep}{\not\!\perp\!\!\!\perp }
\newcommand{\Prb}{\textrm{Pr}}

\DeclareMathOperator*{\argmax}{arg\,max}

\newcommand{\De}{\mathrm{De}}
\newcommand{\PA}{\mathrm{PA}}

\newcommand{\ind}{\mathbf{1}}
\newcommand{\Ext}{\overset{\rightharpoonup}}

\newcommand{\PDAG}{\mathrm{PDAG}}
\newcommand{\MPDAG}{\mathrm{MPDAG}}

\newcommand{\Cl}{\mathrm{Cl}} % generic closure operator

\newcommand{\Inv}{\mathrm{Inv}}
\newcommand{\Chg}{\mathrm{Chg}}
\newcommand{\hInv}{\widehat{\Inv}}
\newcommand{\hChg}{\widehat{\Chg}}

\newcommand{\Sadm}{\mathcal{S}_{\mathrm{a}}}
\newcommand{\Tadm}{\mathcal{T}_{\mathrm{a}}}
\newcommand{\Wit}{\Omega}
\newcommand{\Witu}{\Omega_{u\to v}}

\newcommand{\Gtest}{G_{\mathrm{test}}}
\newcommand{\redGtest}{G_{\Sadm}^\mathrm{test}}

\DeclareMathOperator{\DirE}{\mathrm{DirE}}   % directed-edge set of a (M)PDAG/graph
\DeclareMathOperator{\Dkl}{\mathrm{D}}       % binary KL for Chernoff-style bounds

\title{Scalable Contrastive Causal Discovery under Unknown Soft Interventions}

\author[1,2]{Mingxuan Zhang\textsuperscript{*}}
\author[2,3]{Khushi Desai\textsuperscript{*}}
\author[2,4]{Sopho Kevlishvili}
\author[1,2,3,4,5]{Elham Azizi}

\affil[1]{%
    Department of Systems Biology, Columbia University
}
\affil[2]{%
    Irving Institute for Cancer Dynamics, Columbia University
}
\affil[3]{%
    Department of Computer Science, Columbia University
}
\affil[4]{%
    Department of Biomedical Engineering, Columbia University
}
\affil[5]{%
    Data Science Institute, Columbia University
}

\begin{document}
\maketitle

% co-first footnote, COMMENT OUT DURING SUBMISSION
\renewcommand\thefootnote{*}
\footnotetext{These authors contributed equally.}

\begingroup
\renewcommand\thefootnote{}
\footnotetext{Preprint}
\endgroup

\renewcommand\thefootnote{\arabic{footnote}}

\begin{abstract}
Observational causal discovery is only identifiable up to the Markov equivalence class. While interventions can reduce this ambiguity, in practice interventions are often soft with multiple unknown targets. In many realistic scenarios, only a single intervention regime is observed. We propose a scalable causal discovery model for paired observational and interventional settings with shared underlying causal structure and unknown soft interventions. The model aggregates subset-level PDAGs and applies contrastive cross-regime orientation rules to construct a globally consistent maximal PDAG under Meek closure, enabling generalization to both in-distribution and out-of-distribution settings. Theoretically, we prove that our model is sound with respect to a restricted $\Psi$ equivalence class induced solely by the information available in the subset-restricted setting. We further show that the model asymptotically recovers the corresponding identifiable PDAG and can orient additional edges compared to non-contrastive subset-restricted methods. Experiments on synthetic data demonstrate improved causal structure recovery, generalization to unseen graphs with held-out causal mechanisms, and scalability to larger graphs, with ablations supporting the theoretical results.
\end{abstract}

\section{Introduction}\label{sec:intro}
\vspace{-2mm}
Causal discovery is a fundamental challenge across scientific disciplines, including biology~\cite{sachs2005causal, farh2015genetic, robins2000marginal}, economics~\cite{hoover2008causality}, and climate science~\cite{zhang2011causality}. Causal relationships are often modeled as directed acyclic graphs (DAGs), where nodes represent random variables and directed edges encode causal relationships~\cite{pearl2009causality}. The goal of causal discovery is to recover such graph structures. Observational data alone only allows causal discovery models to be identifiable up to a Markov equivalence class (MEC). Interventional data has been shown to improve identifiability and reduce MEC size~\cite{hauser2014two, kocaoglu2017experimental, ghassami2018budgeted}, though interventions are often assumed to be perfect (hard) with known targets. In many realistic settings (\cite{smith2013resting}, \cite{sachs2005causal},\cite{zhang2019noaa},\cite{clarida2000monetary}), however, only one interventional regime is available and intervention targets are unknown. Additionally, interventions manifest as distributional shifts in a variable's mechanism while leaving the graph structure unchanged. 

~\cite{jaber2020causal} lays the groundwork for structural learning in non-Markovian systems from observational and soft interventional data with unknown targets, augmenting the causal graph with
$F$-nodes to encode intervention effects and introducing the $\Psi$-Markov property to characterize equivalence classes over graphs and targets. Building on this, they propose $\Psi$-FCI, a constraint-based algorithm proven sound and complete for recovering this equivalence class in the sample limit. Despite these theoretical guarantees, $\Psi$-FCI relies on global oracle access to all within-regime conditional independencies and cross-regime invariances implied by the intervention targets. Moreover, it does not scale to large graphs, requires observed environment labels over multiple regimes, and is not designed to generalize to out-of-distribution graphs.

We present \textit{SCONE} (Scalable contrastive Causal discOvery under unknowN soft intervEntions), a scalable framework for learning causal structure from two observational regimes with unknown soft intervention targets that generalizes to out-of-distribution graphs and causal mechanisms. SCONE operates under restricted information: it only accesses (i) finite within-regime subset-level PDAGs on admissible subsets of variables and (ii) finite tested cross-regime invariance queries. Within this setting, SCONE proceeds in three stages: it first estimates local subset-level PDAGs over admissible node sets via classical causal discovery, then applies cross-regime contrastive orientation rules to resolve edge orientations that remain ambiguous within either regime alone, and finally aggregates subset-level constraints into a global maximally oriented PDAG via axial attention. More specifically, our contributions are as follows:
\vspace{-3mm}
\begin{itemize}
    \item \textbf{Model:} We introduce a scalable contrastive causal discovery architecture for two-regime data with unknown soft intervention targets.

    \item \textbf{Theoretical Results:} We formalize subset-restricted $\Psi$-equivalence and the test-induced $\Psi$-essential graph, and prove soundness and consistency of our method in the restricted-information setting.

    \item \textbf{Empirical Results:} We show robust causal structure recovery across diverse benchmarks versus state-of-the-art baselines, and use ablations to isolate the effects of the contrastive orientation rules, contrastive features, and global aggregation module.
\end{itemize}

\section{Related Work}
\vspace{-2mm}
\subsection{Classic Causal Discovery}
\vspace{-3mm}
Traditional causal discovery algorithms search the combinatorial space of DAGs and fall into two broad categories. Constraint-based methods infer causal structure by testing conditional independence relations in the data~\cite{glymour2019review}. Examples include PC and FCI~\cite{spirtes2000causation}, which recover CPDAGs from observational data, and JCI~\cite{mooij2020joint}, which utilizes interventional data. Score-based methods identify causal structure by optimizing a goodness-of-fit score over candidate DAGs (eg. GES~\cite{chickering2002optimal}, GIES~\cite{hauser2012characterization}, CAM~\cite{buhlmann2014cam}, and IGSP~\cite{wang2017permutation}). While foundational, these methods require recursive search and intensive graph operations, making them extremely difficult to scale. 

\vspace{-3mm}
\subsection{Differentiable Causal Discovery}
\vspace{-3mm}
Continuous optimization approaches reformulate causal discovery as a differentiable program over a parameterized graph, bypassing the combinatorial search of traditional methods. NOTEARS~\cite{zheng2018dags} introduces an algebraic acyclicity constraint on the adjacency matrix, enabling gradient-based optimization over the space of DAGs. Subsequent work extends this framework to more expressive parameterizations, including polynomial and neural SEMs~\cite{lee2019scaling, lachapelle2019gradient, zheng2020learning}, alternative acyclicity constraints~\cite{ng2020role, bello2022dagma}, and support for interventional data~\cite{brouillard2020differentiable}. However, differentiable causal discovery methods remain limited in scalability, numerical stability, and cannot generalize to out-of-distribution graph structures. 

\vspace{-3mm}
\subsection{Generalizable and scalable architectures}
\vspace{-3mm}
Foundation model-inspired approaches leverage transformer architectures to improve generalization and scalability. AVICI~\cite{lorch2022amortized} employs a variational inference model with axial attention to predict causal structure from observational or interventional data. SEA~\cite{wu2024sample} samples variable subsets, applies classical causal discovery on subgraphs, and aggregates predictions via axial attention over global summary statistics. However, both methods assume perfect interventions with known targets, limiting their applicability in general interventional settings.

\vspace{-3mm}
\subsection{$\Psi$-Markov equivalence and Environment Invariance Models}
\vspace{-3mm}
Theoretical work~\cite{jaber2020causal} establishes foundations for structural learning from soft interventional data with unknown targets in non-Markovian systems, introducing the $\Psi$-Markov property to characterize equivalence classes and the $\Psi$-FCI algorithm for provable recovery. However, $\Psi$-FCI requires a recursive search over possible graphs, and requires oracle access to all within-regime conditional independence and cross-regime invariances. ICP and its nonlinear extension~\cite{peters2016causal, heinze2018invariant, gamella2020active, pfister2019invariant} exploit mechanistic invariance across environments to identify causal parents, but are limited to parent identification and cannot orient additional edges.

\section{Theoretical Results}
\vspace{-2mm}
\label{theory}
We prove model soundness in a two-regime setting $c\in\{0,1\}$ with shared DAG $G=(V,E)$, where $(P^{(0)},P^{(1)})\in\mathcal M(G)$ is generated by a fixed-graph modular soft-intervention SCM with unknown targets
$I\subseteq V$. The model only accesses (i) subset-level within-regime PDAG summaries
$E_S^{(c)}=\mathcal D(P_S^{(c)})$ for admissible $S\in\Sadm$ and (ii) a finite set of tested invariance literals
$\Inv(v\mid Z)$ for $(v,Z)\in\Tadm$, (exact definitions are given in Appendix~\ref{appendix::ops}). In analogy to $\Psi$-Markov, we define the \emph{restricted $\Psi$-image} and
\emph{restricted $\Psi$-equivalence} induced by this information, and target the corresponding estimand $\Gtest(G)$: the
maximal PDAG (MPDAG) whose edges are compelled over the restricted $\Psi$-equivalence class. Using the two-regime mixture
augmentation with regime label $C$, we interpret invariance as $C$-conditional independence and, under
regime-faithfulness on tested witnesses, propose three examples of contrastive orientation rules and prove  their restricted-$\Psi$-soundness. Finally, we show that aggregating subset patterns and orienting edges based on the rules yields a global MPDAG (post Meek propagation) whose directed edges are a subset of the directed edges in $\Gtest(G)$, and we can recover $\Gtest(G)$ asymptotically. We also show that non-contrastive subset-restricted methods cannot recover a subset of edges we learn.

\subsection{Defining Restricted Equivalence Class and I-EG}
\vspace{-3mm}
Traditional $\Psi$-Markov (and algorithms such as $\Psi$-FCI) ~\cite{jaber2020causal} studies soft interventions with unknown targets under a
\emph{global oracle} constraint family: (i) all within-regime conditional independence and (ii) all cross-regime invariance
implied by the (unknown) target sets. This oracle induces a $\Psi$-equivalence class over augmented objects $\langle G,I\rangle$,
and $\Psi$-FCI is analyzed relative to this full-information class.

Our setting is \emph{subset-restricted}: the model does not observe the full conditional independence structure of $P^{(0)}$ or $P^{(1)}$, nor can it query arbitrary invariance. Instead, it only accesses (a) finitely many within-regime subset level PDAGs
$\{E_S^{(c)}\}_{S\in\Sadm,\ c\in\{0,1\}}$ and (b) finitely many tests for conditional invariance across regimes at node level
$\{\Inv(v\mid Z)\}_{(v,Z)\in\Tadm}$, where $\Inv(v\mid Z):=\ind\{P^{(0)}(X_v\mid X_Z)=P^{(1)}(X_v\mid X_Z)\}$.
We therefore introduce the \emph{restricted $\Psi$-image} $\Psi^{\mathrm{res}}_{\Sadm,\Tadm}(G;\mathcal M)$ (Def.~\ref{res-psi-markov}),
which maps a candidate DAG $G$ to the set of achievable query-outcome tuples under $(P^{(0)},P^{(1)})\in\mathcal M(G)$, and define
restricted $\Psi$-equivalence by equality of these images (Def.~\ref{def:res-psi-equivalence}). Equivalently, the model’s information
is summarized by the restricted constraint family (Def.~\ref{res-psi-markov})
\[
\mathcal F_{\psi}(\Sadm,\Tadm)
\ :=\
\big\{\, E^{(c)}_S(\cdot)\big\}
\ \cup\
\big\{\, \Inv(v\mid Z)(\cdot)\big\}
\]
Since $\mathcal F_\psi(\Sadm,\Tadm)$ is weaker than the oracle family, the induced equivalence is generally coarser, so
$[G]_{\Psi}\subseteq [G]_{\Psi^{\mathrm{res}}}$. Our estimand is the test-induced $\Psi$ essential graph $\Gtest(G)$ (Def.~\ref{def:Gtest});
in other words, under $\mathcal M$ and access restricted to $\mathcal F_\psi(\Sadm,\Tadm)$, the causal output is identifiable up to the
restricted-$\Psi$ equivalence class $[G]_{\mathrm{res}}$.

\subsection{Restricted $\Psi$-Sound Orientation Rules}
\vspace{-3mm}
We propose three restricted $\Psi$-sound contrastive rules that provably orient causal edges that are not compelled by any individual-regime subset PDAGs alone. We construct these rules by attaching tested invariance relationships to local subset PDAG motifs, more specifically regime-sensitive edges (Prop.~\ref{prop:ssi-psi}), candidate colliders (Prop.~\ref{prop:cvt-psi}), and discriminating paths (Prop.~\ref{prop:dpt-psi}). To prove soundness for these rules, we first connect invariance testing to graphical structure by defining the mixture
distribution $P^*(C=c,X):=\pi_c P^{(c)}(X)$ over the DAG nodes and the binary label $C$. Under this construction, two-regime conditional invariance is 
$C$-conditional independence:
\[
\Inv(v\mid Z)=1 \ 
\ \Longleftrightarrow\ C\indep X_v\mid X_Z \ \text{under } P^*,
\]
and $\Chg(v\mid Z)=1$ corresponds to $C\notindep X_v\mid X_Z$ (Lemma~\ref{lem:inv-ci}). We then construct an augmented graph
$G^*$ obtained by adding $C$ with edges $C\to t$ for intervention targets $t\in I$, and assume \emph{regime-faithfulness on tested
witnesses}, which guarantees that for all admissible witness queries $(v,Z)\in\Tadm$, the statements $C\indep X_v\mid X_Z$ and
$C\notindep X_v\mid X_Z$ agree with the corresponding $d$-separation relations $C\perp_d v\mid Z$ and $C\not\!\perp_d v\mid Z$
in $G^*$ (Assumption~\ref{ass:change-faithfulness}).

Here we formally define the contrastive orientation rules and give a brief proof sketch for each. 
\paragraph{Single Sided Invariance} This rule targets edges that remain undirected in the within-regime subset-level PDAGs, and orients them only when the node pair exhibits asymmetric cross-regime behavior
\begin{proposition}
Let $S\in\Sadm$ with $\{i,j\}\subseteq S$ such that $i-j$ is an undirected adjacency in both population local PDAGs
$E_S^{(0)}$ and $E_S^{(1)}$ (Definition~\ref{def::pdag}). Let $Z:=S\setminus\{i,j\}$ such that
$(j, Z), (i, Z) \in \Tadm$ and $Z \in \mathcal{W}(j)$, $Z \in \mathcal{W}(i)$.
Assume moreover that $i$ and $j$ are adjacent in every $G'\in [G]_{\mathrm{res}}$ (equivalently, $i-j$ is in the
skeleton of $\Gtest(G)$). If
\[
\Chg(j\mid Z)=1,
\qquad
\Inv(i\mid Z)=1,
\]
then $i\to j$ is invariant across the restricted $\Psi$-equivalence class and thus is oriented identically in $\Gtest(G)$.
\end{proposition}

The full proof of the proposition is in the Appendix (Proof for Prop.~\ref{prop:ssi-psi}). In summary, we rule out the reverse orientation by showing it would propagate regime dependence in the wrong way. Introduce an environment node $C$ so that our invariance/change tests become statements about whether $C$ is separated from / connected to a node under the witness conditioning set $Z$ (Assumption~\ref{ass:change-faithfulness}). If we oriented the unresolved edge as $j\to i$, any $Z$-active connection from $C$ to $j$ would extend through $j\to i$, forcing $i$ to also be $C$-dependent. 

\paragraph{Contrastive V-structure} This rules extends contrastive comparison to undirected triplets across regimes
\begin{proposition}
Let $S\in\Sadm$ with $\{i,j,k\}\subseteq S$ such that $i-j-k$ is an unshielded triple in the local skeleton on $S$.
Let $Z:=S\setminus\{i,j,k\}$ and assume $(i,Z),(j,Z),(k,Z)\in\Tadm$ with $Z\in\mathcal W(i)\cap\mathcal W(j)\cap\mathcal W(k)$.
Assume moreover that $i-j$ and $j-k$ are edges in the common skeleton of $[G]_{\mathrm{res}}$ (equivalently, in the skeleton of $\Gtest(G)$).
If
\[
\Chg(j\mid Z)=1,\qquad \Inv(i\mid Z)=1,\qquad \Inv(k\mid Z)=1,
\]
then $i\to j\leftarrow k$ is compelled across $[G]_{\mathrm{res}}$ and hence appears in $\Gtest(G)$.
\end{proposition}

The full proof is deferred to the Appendix (Proof for Prop.~\ref{prop:cvt-psi}). In this proof, we rule out any non-collider orientation at $j$ by showing it would create regime dependence at $j$, contradicting $\Inv(j\mid Z)=1$. We introduce an environment node $C$ so that invariance corresponds to $C$ being separated from $j$ given the witness set $Z$ (Assumption~\ref{ass:change-faithfulness}). If $i-j-k$ were oriented with a tail at $j$, then any $Z$-active connection from $C$ reaching one neighbor would extend through that tail into $j$, forcing $C$ to be $Z$-connected to $j$ and hence $\Chg(j\mid Z)=1$. Therefore $j$ must be a collider, yielding $i\to j\leftarrow k$, and the same argument applies across the restricted $\Psi$-equivalence class, so the v-structure is compelled and appears in $\Gtest(G)$.

\paragraph{Contrastive Discriminating Path} We can further
extend the logic to operate on discriminating paths(Def. ~\ref{def:candidate-dp-motif}) in subset level PDAGs. We point readers to Appendix (Lemma~\ref{lem:active-concat}, Lemma ~\ref{lem:dpt-first-mismatch}, Lemma ~\ref{lem:dpt-j-to-r}, Lemma~\ref{lem:dpt-contradiction}, Prop. ~\ref{prop:dpt-psi}) for the exact formulation and proof due to length constraints. In summary, the proof decomposes into three steps: we first show that any alternative replay assignment(Def. !\ref{def:replay-status}, Def. ~\ref{def:replay-nodes}) that changes the discriminating-path mark(Def.~\ref{def:dp-mark}) must have a \emph{first replay disagreement} node along the path (Lemma~\ref{lem:dpt-first-mismatch}); we then prove that, under some $j$ where $\Chg(j\mid Z^\pi)=1$ and the path-segment compatibility conditions, this first mismatch inherits regime dependence by concatenating an unblocked path from the regime node $C$ to $j$ with the corresponding unblocked segment of the discriminating path, yielding $C\not\!\perp_d r\mid Z^\pi$ (via Lemma~\ref{lem:active-concat} and Lemma~\ref{lem:dpt-j-to-r}); finally, Assumption~\ref{ass:change-faithfulness} and Lemma~\ref{lem:inv-ci} convert this graphical dependence into $\Chg(r\mid Z^\pi)=1$, contradicting the replay-node invariance premise $\Inv(r\mid Z^\pi)=1$ (Lemma~\ref{lem:dpt-contradiction}). The proposition then lifts this uniqueness argument from the generating DAG to every $G'\in [G]_{\mathrm{res}}$ by restricted $\Psi$-equivalence, implying that the discriminating-path mark on $(j,k)$ is compelled and therefore appears in $\Gtest(G)$.

\vspace{-3mm}
\subsection{Global Restricted $\Psi$ I-EG Estimation}
\vspace{-3mm}
In this section we show soundness of our model by: (1) global contrastive aggregation guarantee and provable gap between non-contrastive and contrastive aggregation (2) asymptotic recovery of the test-induced $\Psi$ essential graph. The full proof is given in Appendix~\ref{thm:ctr-dom-soundness}. For (1) we show that any subset aggregator without cross-regime contrasts cannot recover additional edges compelled by contrastive information, necessitating our contrastive framework. Since non-contrastive aggregators do not have access to tested invariance queries, we define a reduced estimand $\redGtest$ based only on per-regime local PDAGs on admissible subsets, and characterize a gap $R$ between this reduced estimand and the restricted test-induced $\Psi$ essential graph $\Gtest$. We formulate the following theorem:

\begin{theorem}[Contrastive aggregation guarantee and $\Gtest$-soundness]\label{thm:main-ctr-dom}
For a testing family $(\Sadm, \Tadm)$ with local PDAGs $\{ E_S^{(c)}\}_{S \in \Sadm, c \in \{ 0, 1 \}}$ and queries $\{\Inv(v \mid Z)\}_{(v, Z) \in \Tadm}$, under Assumptions~\ref{ass:change-faithfulness}, ~\ref{ass:closure-sound}, and assuming $\DirE(H_\mathrm{per}) \subseteq \DirE(\redGtest)$:
\begin{enumerate}
    \item \textbf{Monotonic enrichment:} $\DirE(H_\mathrm{obs}) \subseteq \DirE(H_\mathrm{per}) \subseteq \DirE(H_\mathrm{ctr})$.
    \item $\Gtest$\textbf{-soundness:} $\DirE(H_\mathrm{ctr}) \subseteq \DirE(\Gtest)$.
    \item \textbf{Separation:} If $R \neq \varnothing$ then $\DirE(\mathcal{A}) \subsetneq \DirE(\Gtest)$ and $\DirE(H_\mathrm{per}) \subsetneq \DirE(\Gtest)$.
\end{enumerate}
\end{theorem}

Under the assumption of per-regime soundness under $\redGtest$, we show monotonic enrichment of global observational, per-regime, and contrastive MPDAGs based on monotonicity of the $\MPDAG(\cdot)$ operator on increasing knowledge sets on directed edges $\DirE(\cdot)$ of $\redGtest$ and those produced by contrastive rules. We then show soundness of the contrastive MPDAG by extension of the reduced $\Psi$-image with contrastive rules that are restricted $\Psi$-sound by Propositions 1, 2, and 3. Finally, we establish that any edge in $R$ is unrecoverable by any non-contrastive subset-respecting aggregator (Lemma~\ref{lem:noncontrastive-impossible}, Corollary~\ref{cor:red-ident}), since the recovery of maximal sound directed-edges from per-regime information alone is bounded by $\redGtest$. 

Next we formulate the theorem for asymptotic recovery of $\Gtest$.

\begin{theorem}[Consistency for the test-induced restricted $\Psi$ essential graph]\label{thm:main-psi}
Let $\widehat G$ be the model output after $T$ sampled subsets:
(i) estimate $\widehat E_S^{(c)}$ on each sampled $S$ and each regime $c$,
(ii) estimate $\hInv(v\mid Z)$ for tested $(v,Z)$ pairs on sampled witnesses,
(iii) apply orientation rules using these estimates, and
(iv) aggregate all inferred constraints with closure $\Cl$ to output $\widehat G$.
Under Assumptions~\ref{ass:finite-scope}--\ref{ass:closure-sound} and witness coverage
(Assumption~\ref{ass:witness-coverage-psi}), as $n,T\to\infty$,
\[
\Prb\!\left(\widehat G = \Gtest\right)\ \to\ 1.
\]  
\end{theorem}

The full proof is given in the Appendix (\ref{thm:main-psi}). In this proof, we show that as $n \to \infty$, $\Prb(\mathcal E_n) \to 1$, where $\mathcal E_n$ is the event that all invoked local PDAGs and invariance tests used by the model are correct over the finite admissible families. On $\mathcal E_n$, since the empirical restricted $\Psi$-image coincides with the population restricted $\Psi$-image, the estimated graph $\widehat G$ lies in the same restricted $\Psi$-equivalence class as $G$. By $\Gtest$-soundness, this implies $\DirE(\widehat G) \subseteq \DirE(\Gtest)$. For any compelled edge $e \in \Gtest$, witness coverage ensures it is sampled with probability $\to 1$. Therefore, each compelled edge is inserted into the aggregation at the population level. By closure soundness, no incorrect orientations are introduced. Applying a union bound over the finite edge set, yields $\Pr(\widehat G = \Gtest) \to 1$ as $n, T \to \infty$.

\section{Model Architecture}
\vspace{-2mm}
The SCONE architecture operates using two streams: a marginal stream processing edge-level representations from admissible subsets, and a global stream maintaining dense representations over all ordered node pairs. The architecture integrates classical causal discovery, contrastive orientation bias heads, and axial attention for scalable aggregation.

The method that is most similar to our architecture is SEA~\cite{wu2024sample}; however, our architecture remains significantly distinctive in the following ways.

\subsection{Constructing marginal causal structures}
\vspace{-3mm}
\label{marginal}
Let $V = \{ 1, ..., N\}$ be the nodes. We form $T$ admissible subsets $S_1, ..., S_T \subseteq V$ and fix a global catalog of undirected edges $e_k = \{ i, j\}$ for $k = 1,..., K$ unique unordered pairs. We consider two regimes $c \in \{ 0, 1\}^N$.

\vspace{-3mm}
\subsubsection{Sampling admissible subsets}
\vspace{-3mm}
We construct admissible subsets via a greedy sampling strategy that prioritizes soft intervention targets while ensuring graph coverage. Node sensitivity scores $s_i \in [0,1]$ combine standardized mean shift across regimes and mutual information with the regime label. Pairwise contrast scores 
\begin{equation}
\Gamma_{i,j} = |\mathrm{corr}^{(1)}(X_i,X_j) - \mathrm{corr}^{(0)}(X_i,X_j)|
\end{equation}
capture cross-regime correlation shifts. Both are decayed by visit frequency, $n_{i,j}^t$, for subset $t$ to discourage resampling:
\begin{equation}
\alpha^{\,t}_{i,j} = \alpha_{i,j}/q^{\,n^{\,t}_{i,j}}, \qquad \Gamma^{\,t}_{i,j} = \Gamma_{i,j}/r^{\,n^{\,t}_{i,j}},
\end{equation}
where $q, r > 1$ control exploration-exploitation. A seed node is drawn categorically, then nodes are added greedily by scoring candidates on pairwise relevance $\alpha^{\,t}_{i,j}$, sensitivity $s_j$, mean contrast $\Gamma^{\,t}_{i,j}$, and a coverage bonus. Visit and pair counts are updated after each subset is formed (full equations in Appendix~\ref{appendix:arch-sampling}).

\vspace{-3mm}
\subsubsection{Classical causal discovery ensemble}
\vspace{-3mm}
The ensemble causal discovery module combines multiple base learners to produce robust edge votes over four classes: $\{i\!\to\!j,\ j\!\to\!i,\ i-j,\ \text{no-edge}\}$ for each unordered pair $\{i,j\}$. We primarily use PolyBIC as a pairwise functional learner (and provide the option to use GES-BIC~\cite{} as a structure-based learner), which orients edges by comparing temperature scaled BIC scores of polynomial regressions to weight either direction.
Bootstrap replicates $\{x^{(r)}\}_{r=1}^{R_\mathrm{boot}}$ are drawn by resampling $S$ observations with replacement. Each replicate yields scores decoded into one-hot votes via a threshold $\tau_\ell$ and tie-breaking margin $\epsilon$, weighted vote counts are accumulated across learners and replicates, and the final edge class is determined by majority vote with no-edge as the tie-breaking default (full details in Appendix~\ref{appendix:arch-ensemble}).

\vspace{-3mm}
\subsubsection{Edge tokenization}
\vspace{-3mm}
For each sampled subset $S_t$, the ensemble CD module runs independently on both regimes to produce local PDAGs, whose edges are collected and deduplicated into the candidate edge set $K_\mathrm{cand}$ of $K$ edges. Each edge $e=\{i,j\}$ is encoded with a 6-dimensional endpoint encoding, one-hot node identifiers, and per-regime pairwise statistics $\phi_{ij}^{(c)} \in \mathbb{R}^4$ — containing correlation, partial correlation, and pairwise regression coefficients — forming $x_{e,c}$. These are projected to dimension $d$ and augmented with sinusoidal positional embeddings to yield $H_E \in \mathbb{R}^{B \times T \times K \times 2 \times d}$:
\begin{equation}
H_E[t, e, c, :] = \mathrm{MLP}(x_{e,c}) + \mathrm{PE}_T(t) + \mathrm{PE}_K(k).
\end{equation}
In parallel, a global stream $H_{\rho} \in \mathbb{R}^{B \times N \times N \times d}$ is constructed from the inverse covariance matrix averaged across regimes, projected to dimension $d$ and summed with learned node positional embeddings (full details in Appendix~\ref{appendix:arch-tokenization}).

\vspace{-3mm}
\paragraph{Invariant/contrast reparametrization}
Edge embeddings are then reparameterized into invariant and contrast channels $z_{\text{avg}} = \frac{1}{2}(H_E^{(1)} + H_E^{(0)})$ and $z_{\Delta} = H_E^{(1)} - H_E^{(0)}$, which replace the original regime embeddings as the two operative channels and forms $\widetilde{H}_E$. This bijective reparameterization disentangles regime-invariant structure from regime-specific variation, sharpening the contrast signal available to the orientation rules.

\vspace{-3mm}
\subsection{Bias heads for contrastive orientation}
\vspace{-3mm}
\label{bias-heads}
Each contrastive orientation rule produces signed biases over the three orientation classes $\{i\!\to\!j,\ j\!\to\!i,\ i - j\}$ for each subset $S_t$. To quantify how much a node's mechanism shifts across regimes, we define an invariance score
\begin{equation}
\gamma_v = \mathrm{Norm}\!\left(\sum_{e\ni v}\big\|\Delta_{e}\big\|\right)\in[0,1],
\quad g_v := 1-\gamma_v,
\label{eq:inv-score}
\end{equation}
where $\Delta_{ij} = \widetilde H_E[t,\{i,j\},1,:]-\widetilde H_E[t,\{i,j\},0,:]$ is the regime contrast for edge $e=\{i,j\}$, and $g_v$ denotes its complement. Nodes with high $g_v$ are more likely regime-invariant (non-targets) and thus provide stronger orientation signal.

\vspace{-3mm}
\subsubsection{Single-Sided Invariance (SSI)}
\vspace{-3mm}
For each edge $\{i,j\}$, parent contexts $\mathbf{p}_i, \mathbf{p}_j \in \mathbb{R}^d$ are formed by mean-pooling embeddings of incident candidate edges appearing in at least a fraction $\rho$ of subsets. A shared predictor $\psi$ maps each endpoint embedding and its parent context to a sufficient-statistic vector $\phi^{(c)}_v = \psi([\widetilde{H}_E[t,\{v,\cdot\},c,:]\ \|\ \mathbf{p}_v\ \|\ \mathbf{m}_v])$, and $\delta_v = \|\phi^{(1)}_v - \phi^{(0)}_v\|$ measures the cross-environment shift. A bias toward $i\!\to\!j$ is applied when $j$ shifts more than $i$:
\begin{equation}
\begin{aligned}
b^{\text{SSI}}_{t,\{i,j\},c}(i\!\to\! j) \mathrel{+}= \kappa\cdot\mathrm{clip}(\delta_j-\delta_i,-\tau,\tau),
\\
b^{\text{SSI}}_{t,\{i,j\},c}(j\!\to\! i) \mathrel{-}= \kappa\cdot\mathrm{clip}(\delta_j-\delta_i,-\tau,\tau).
\end{aligned}
\end{equation}
gated by $\gamma_j \geq \gamma_i + \eta$ and $\gamma_i \leq \gamma_{\max}$, ensuring the bias fires only when $j$ is the more perturbed endpoint (full details in Appendix~\ref{appendix:arch-ssi}).

\vspace{-3mm}
\subsubsection{Contrastive V-structure (CVT)}
\vspace{-3mm}
For each unshielded triple $(i,j,k)$ in $S_t$, we score the likelihood of a collider $i\!\to\! j\!\leftarrow\! k$ using edge embeddings and cross-regime contrast statistics $\Delta_{ij}, \Delta_{kj}$:
\begin{equation}
\small
\begin{aligned}
s^{(1)}_{\mathrm{vstr}}(i\!\to\! j\!\leftarrow\! k) = 
\sigma\!\Big(&u^\top\big[\widetilde H_E[t,\{i,j\},1,:]\ \|\ \\
&\widetilde H_E[t,\{k,j\},1,:]\
\|\
\Delta_{ij}\ \|\ \Delta_{kj}\big]\Big)
\end{aligned}
\end{equation}
where $\sigma$ is an activation and $u$ is learned. Scaled by the invariance score $g_j$, this score applies a positive bias toward arrowheads at $j$ on both incident edges, promoting the collider orientation while penalizing the reverse directions (full details in Appendix~\ref{appendix:arch-cvt}).

\vspace{-3mm}
\subsubsection{Contrastive Discriminating Path (DPT)}
\vspace{-3mm}
Let $\mathcal{G}_t$ be the \textit{edge-graph} of subset $S_t$, where nodes represent edges and adjacency encodes shared vertices. For each edge $\{i,j\}$, the top-$B$ short paths $\pi=(e_0,\dots,e_m)$ of length $m\le L$ are enumerated via BFS. Each path is encoded independently per environment via a DeepSets-style aggregator $\mathrm{DeepSet}(\cdot)$ over edge embeddings and cross-environment contrasts, producing per-environment scores $s^{(c)}_\pi \in [0,1]$. The higher-confidence environment $c^* = \arg\max_c\, s^{(c)}_\pi$ provides the orientation signal, which is applied to the other environment $\bar{c} = 1 - c^*$:
\begin{equation}
\begin{aligned}
b^{\text{DPT}}_{t,\{i,j\},\bar{c}}(i\!\to\! j) \mathrel{+}= \lambda\, s^{(c^*)}_\pi,\\
b^{\text{DPT}}_{t,\{i,j\},\bar{c}}(j\!\to\! i) \mathrel{-}= \lambda\, s^{(c^*)}_\pi,
\end{aligned}
\end{equation}
aggregated across candidate paths (full details in Appendix~\ref{appendix:arch-dpt}).

\subsection{Aggregation model architecture}
\label{aggregation}
\vspace{-3mm}
SCONE maintains two streams: a $\emph{marginal}$ stream $H_E \in \mathbb{R}^{T \times K \times 2 \times d}$ over edge tokens from sampled subsets, and a \emph{global} stream $H_\rho \in \mathbb{R}^{N \times N \times d}$ over all ordered node pairs. Edge embeddings $H_E$ are first reparametrized to $\widetilde{H}_E$ and the two streams are coupled across layers via message passing (full details in Appendix~\ref{appendix:arch-aggregator}).

\vspace{-3mm}
\subsubsection{Axial Aggregator}
\vspace{-3mm}
\paragraph{Marginal axial block.}
Each block applies axial attention along $T$ and $K$, followed by a feedforward layer with pre-LN and residual connections. Contrastive bias heads (CVT, SSI, DPT) are applied after attention and gated by per-edge invariance scores $\gamma_v$ before being projected back into the embedding space.

\vspace{-3mm}
\paragraph{Message Passing.}
For marginal to global update, at each layer, edge embeddings are pooled across subsets via learned attention weights $w_t$, producing a per-edge message $m_e \in \mathbb{R}^d$. These are scattered into a global node-pair table by writing $H_\rho[i,j,:] \leftarrow m_e$ for each candidate edge $e = \{i,j\}$, with unobserved pairs defaulting to zero. In the global block, the populated node-pair table undergoes row- and column-wise axial attention over the $N \times N$ grid, followed by a position-wise feedforward layer, enabling global relational reasoning across all node pairs. For the global to marginal update, the updated global representations are broadcast back to the marginal stream by adding the corresponding node-pair embedding to each edge token:
\begin{equation}
\widetilde{H}_E[t, e, :, :] \leftarrow H_E[t, e, :, :] + H_\rho[i, j, :]
\end{equation}
This closes the local-global loop, ensuring subset-level edge representations are informed by full graph context at each layer.

\vspace{-3mm}
\subsubsection{Global head}
\vspace{-3mm}
Since candidate edges may not capture all edges required for orientation, an auxiliary global head applies a linear projection to $H_\rho$ to predict edge classes over all ordered node pairs, providing supervision beyond the candidate set. Its loss $\mathcal{L}_{\text{global}}$ is computed over all positive edges and a sampled subset of negatives, weighted by $\lambda_{\text{global}}$ in the total objective, and its logits can optionally be fused with marginal logits for non-candidate edges at inference. Crucially, this auxiliary signal encourages consistency between local subset-level predictions and global pairwise representations (full details in Appendix~\ref{appendix:arch-global-head}.

\vspace{-3mm}
\subsubsection{Optimization objective}
\vspace{-3mm}
The model is trained with a cross-entropy loss over candidate edges with labels $\{u\!\to\!v,\ v\!\to\!u,\ \text{no-edge}\}$, supplemented by a soft SHD penalty and a pairwise ranking loss:
\begin{equation}
\mathcal{L} = \mathcal{L}_{\text{CE}} + \lambda_{\text{SHD}}\,\mathcal{L}_{\text{SHD}} + \lambda_{\text{rank}}\,\mathcal{L}_{\text{rank}} + \lambda_{\text{global}}\,\mathcal{L}_{\text{global}},
\end{equation}
where $\mathcal{L}_{\text{SHD}}$ penalizes the squared deviation of predicted edge probabilities from ground truth, $\mathcal{L}_{\text{rank}}$ is a pairwise hinge loss ranking true directed edges above non-edges, and $\mathcal{L}_{\text{global}}$ is the auxiliary global head loss over all node pairs. A gate regularization term encourages contrastive bias heads to remain active throughout training, preventing collapse to per-regime local orientations. The model is optimized with AdamW with a warmup and cosine learning rate schedule, and EMA weights are used at evaluation (full details in Appendix~\ref{appendix:arch-opt}).

\section{Empirical Results}
\vspace{-2mm}
We evaluate SCONE on synthetic soft-intervention datasets with unknown intervention targets across different experimental settings, comparing against state-of-the-art baselines. Details of synthetic data generation, evaluation metrics, baseline configurations and extended results are provided in Appendices ~\ref{appendix:datagen}, ~\ref{appendix:metrics}, and ~\ref{appendix:fullbench}.

\vspace{-3mm}
\subsection{Baseline Models}
\vspace{-3mm}
We consider two types of baseline models: (1) scalable CD learners with out-of-distribution generalization capability, (2) CD learners that need to learn on data generated by individual causal graphs. Baseline models in (2) are required to be trained on testing datasets, hence, corresponding to a less demanding setting. Importantly, none of these baselines are designed for unknown soft interventions. Methods supporting interventional data are trained on interventional settings with known intervention targets.

\vspace{-3mm}
\paragraph{In-distribution CD learners} \textbf{DCDI} learns a nonlinear structural equation model by jointly optimizing the adjacency matrix and neural mechanisms under a differentiable acyclicity constraint. \textbf{NOTEARS-NONLINEAR} applies continuous optimization framework to nonlinear functions, enforcing DAG structure using a smooth acyclicity constraint. \textbf{DCD-FG} scales DCDI by factorizing the adjacency matrix into low-rank components. 

\vspace{-3mm}
\paragraph{Out-of-distribution generalization capable learners}
\textbf{AVICI} performs amortized inference over causal graphs, learning to approximate $P(G|D)$ across a distribution of data-generating processes. \textbf{SEA} samples variable subsets, runs a standard CD algorithm on each subgraph (FCI or GIES), and aggregates local predictions via axial-attention, incorporating global statistics to produce a full-graph estimate. 

\vspace{-3mm}
\subsection{Results}
\vspace{-3mm}
Extended results for all settings are given in Appendix~\ref{appendix:extresults}. Note that we exclude NOTEARS and DCDI from the 50-node and 100-node experiments since they do not scale to larger graphs. On a single NVIDIA A100 (40GB), runtime per 20-node 20-edge graph ranged from 2 to 96 hours for NOTEARS and 5 to 72 hours for DCDI, depending on initialization and convergence behavior. At 50 and 100 nodes, both methods exceed practical time limit.

\begin{figure}[H]
    \centering
    \includegraphics[width=\linewidth]{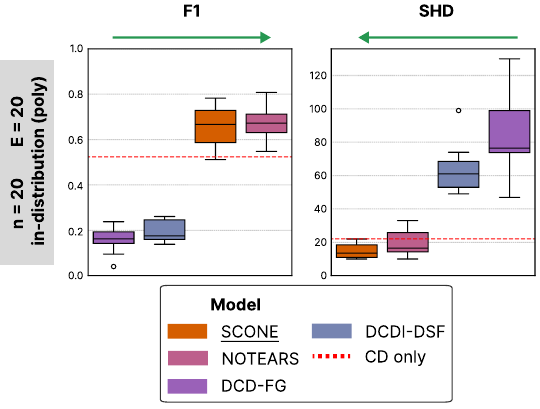}
    \caption{F1 (harmonic mean of precision and recall) and SHD (structural hamming distance) for state-of-the-art methods that only perform in distribution graph prediction on 20 nodes 20 edges. 10 graphs were generated using the polynomial mechanism. The horizontal red dashed line indicates mean across 10 graphs on predictions from running classic causal discovery baseline.}
    \vspace{-5mm}
    \label{fig:indist-results}
\end{figure}

\begin{figure}[H]
    \centering
    \includegraphics[width=\linewidth]{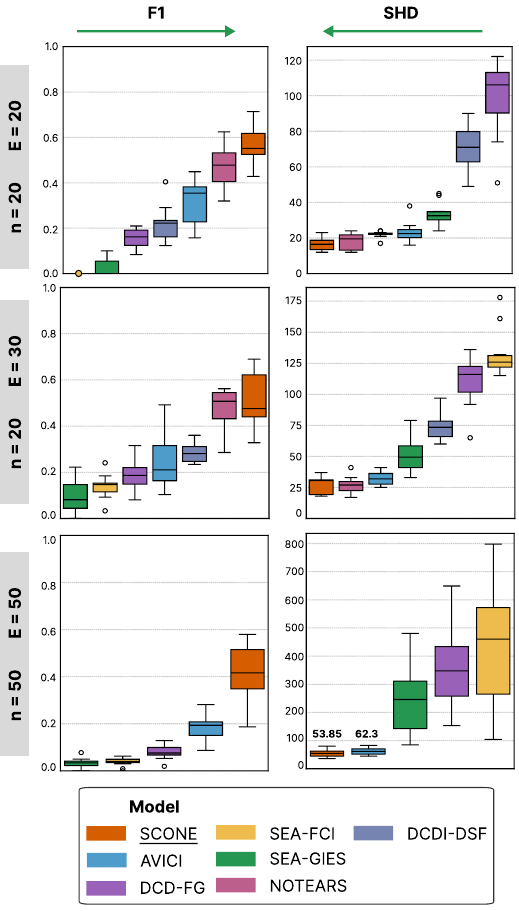}
    \caption{F1 and SHD for state-of-the-art methods compared to SCONE on predictions over 10 graphs with 20 nodes 20 edges, 20 nodes 30 edges, and 20 graphs for 50 nodes 50 edges. SEA-FCI obtains an F1 score of 0.0 in the 20 node 20 edge example, and is represented by a circle on the bottom of the axis.}
    \vspace{-5mm}
    \label{fig:ood_results}
\end{figure}

\vspace{-3mm}
\subsubsection{In-distribution Benchmarking}
\vspace{-3mm}
We evaluate all methods in-distribution, where train, validation, and test data are generated using the Polynomial mechanism only, averaged over 10 graphs with 20 nodes and 20 edges. SCONE achieves the lowest SHD ($14.6 \pm 4.4$) across all baselines (Figure~\ref{fig:indist-results}), indicating improved structural recovery when training and test mechanisms match (extended results in Appendix: Figure~\ref{fig:in-dist-appendix}, Table~\ref{tab:results_poly_indist}).

\vspace{-3mm}
\subsubsection{Out-of-distribution Generalization Benchmarking}
\vspace{-3mm}
We next evaluate out-of-distribution generalization by training on Linear, NN, and NN Additive mechanisms and testing on unseen types (Polynomial, Sigmoid), with varying graph size and edge density to assess robustness across sparsity regimes. SCONE consistently outperforms both in-distribution learners and out-of-distribution baselines, with lower SHD and higher F1 across all evaluated regimes (Figure~\ref{fig:ood_results}), demonstrating strong generalization to unseen causal mechanisms (extended results in Appendix: Figure~\ref{fig:ood-appendix}, Table~\ref{tab:results_main}).

\vspace{-3mm}
\subsubsection{Scalability to Larger Graphs}
\vspace{-3mm}
We evaluate scalability on graphs with 100 nodes and 100 edges. SCONE achieves the lowest SHD and highest F1 at this scale (Table~\ref{tab:resultsscale}), maintaining stable performance as graph size increases, while DCD-FG and SEA variants exhibit significant SHD increase (extended results in Table~\ref{tab:resultsscale-appendix}).

\begin{table}[H]
\centering
\begin{tabular}{lcc}
\toprule
Model & SHD $\downarrow$ & F1 $\uparrow$ \\
\midrule
\textsc{\textbf{Scone}}      & $\textbf{126.7} \pm 25.0$ & $\textbf{0.237} \pm 0.046$ \\
\textsc{Avici}      & $131.3 \pm 18.0$ & $0.100 \pm 0.042$ \\
\textsc{Dcd-Fg}     & $1001.3 \pm 537.9$ & $0.052 \pm 0.011$ \\
\textsc{Sea} (FCI)  & $3178.8 \pm 425.2$ & $0.022 \pm 0.004$ \\
\textsc{Sea} (GIES) & $1797.2 \pm 296.6$ & $0.022 \pm 0.004$ \\
\bottomrule
\end{tabular}
\caption{Comparison of causal discovery methods on graphs with $N=100$ nodes and $E=100$ edges. Results are averaged over 20 graphs (mean $\pm$ std).}
\vspace{-5mm}
\label{tab:resultsscale}
\end{table}

\vspace{-3mm}
\subsubsection{Ablation Study}
\vspace{-3mm}
In our model architecture (Appendix~\ref{appendix::arch}), contrastive orientation rules are implemented as learned biases from independent modules (bias heads) that encode the corresponding structural motifs using cross-regime contrast signals. We also include multiple contrastive features -- such as reparametrization, contrastive sampling, and edge embedding pairwise statistics -- as proxies for cross-regime invariance. Hence, we provide ablation study to empirically validate these theoretical claims.
\paragraph{Ablating Bias Head}
We first ablate the bias heads to empirically validate their contribution to orientation accuracy (Table ~\ref{tab:resultsabl-sconenb}). Removing the bias heads (\textsc{SCONE-NB}) increases SHD and reduces F1, which indicates that contrastive rules aligned additional true edges. We also observed larger gaps at 50 nodes, indicating increasing importance of contrastive orientation in larger graphs (extended results in Table~\ref{tab:resultsabl-sconenbapp}).

\vspace{-3mm}
\paragraph{Ablating Contrastive Features} We ablate the contrastive features to empirically validate that invariance/contrast is necessary to learn additional edges in our setting (Table ~\ref{tab:resultsabl-sconenc}). Removing these features (\textsc{SCONE-NC}) significantly increases SHD and reduces F1 across graph sizes (extended results in Table~\ref{tab:resultsabl-sconencapp}), demonstrating the necessity of contrastive invariance literals for accurate causal structure recovery.

\begin{table}[H]
\centering
\begin{tabular}{lllcc}
\toprule
$N$ & $E$ & Model & SHD $\downarrow$ & F1 $\uparrow$ \\
\midrule
& & \textsc{\textbf{Scone}} & $\textbf{39.0} \pm 9.7$ & $\textbf{0.469} \pm 0.143$ \\
\raisebox{1.5ex}{20} & \raisebox{1.5ex}{40} & \textsc{Scone-NB} & $39.6 \pm 9.0$ & $0.453 \pm 0.123$ \\
\midrule
& & \textsc{\textbf{Scone}} & \textbf{53.85} $\pm$ 11.60 & \textbf{0.424} $\pm$ 0.102 \\
\raisebox{1.5ex}{50} & \raisebox{1.5ex}{50} & \textsc{Scone-NB} & 56.60 $\pm$ 10.79 & 0.393 $\pm$ 0.092 \\
\bottomrule
\end{tabular}
\caption{Comparison of SCONE with and without using orientation rules. Results are averaged over 20 graphs (mean $\pm$ std). SCONE-NB is the model with bias head ablated and thus do not incur rules.}
\vspace{-5mm}
\label{tab:resultsabl-sconenb}
\end{table}

\begin{table}[H]
\centering
\begin{tabular}{lllcc}
\toprule
$N$ & $E$ & Model & SHD $\downarrow$ & F1 $\uparrow$ \\
\midrule
& & \textsc{\textbf{Scone}} & $\textbf{39.0} \pm 9.7$ & $\textbf{0.469} \pm 0.143$ \\
\raisebox{1.5ex}{20} & \raisebox{1.5ex}{40} & \textsc{Scone-NC} & $48.6 \pm 9.3$ & $0.341 \pm 0.136$ \\
\midrule
& & \textsc{\textbf{Scone}} & $\textbf{53.85} \pm 11.60$ & $\textbf{0.424} \pm 0.102$ \\
\raisebox{1.5ex}{50} & \raisebox{1.5ex}{50} & \textsc{Scone-NC} & $66.55 \pm 9.1$ & $0.322 \pm 0.076$ \\
\bottomrule
\end{tabular}
\caption{Comparison of SCONE with and without using contrastive features. Results for $N=20$ are averaged over 10 graphs and $N=50$ over 20 graphs (mean $\pm$ std). \textsc{Scone-NC} is the model with contrastive features ablated.}
\vspace{-5mm}
\label{tab:resultsabl-sconenc}
\end{table}

\section{Conclusion}
\vspace{-2mm}
We introduce \textsc{SCONE}, a scalable contrastive framework for causal discovery from two regimes with unknown soft interventions, targeting the restricted-$\Psi$ essential graph induced by available subset-level PDAGs and admissible invariance tests. We formalize restricted $\Psi$-equivalence and its essential-graph estimand, prove the contrastive rules (SSI, CVT, DPT) are restricted-$\Psi$-sound and only orient edges compelled over the restricted class, and establish a strict separation from non-contrastive subset-respecting aggregators. Under local PDAG/invariance-test consistency, regime faithfulness on tested witnesses, and witness coverage, \textsc{SCONE} asymptotically recovers the restricted-$\Psi$ essential graph. Empirically, it improves structural recovery across mechanism shift, density, and scale, with ablations showing contrastive bias heads and invariance/contrast reparameterization are necessary for orientations beyond per-regime information. Limitations include dependence on faithfulness and positive witness mass, restriction to two regimes, and an information ceiling from the admissible subset/test families; extending to richer interventions, adaptive witness selection, and imperfect regimes is future work.

%\begin{contributions} % will be removed in pdf for initial submission 
					  % (without ‘accepted’ option in \documentclass)
                      % so you can already fill it to test with the
                      % ‘accepted’ class option
    %Briefly list author contributions. 
    %This is a nice way of making clear who did what and %to give proper credit.
    %This section is optional.

    %H.~Q.~Bovik conceived the idea and wrote the paper.
    %Coauthor One created the code.
    %Coauthor Two created the figures.
%\end{contributions}

\section*{Competing Interests}
M.Z., K.D., S.K., E.A. are inventors on a provisional patent application filed by the Trustees of Columbia University in the City of New York.

\begin{acknowledgements} % will be removed in pdf for initial submission,
						 % (without ‘accepted’ option in \documentclass)
                         % so you can already fill it to test with the
                         % ‘accepted’ class option
  We thank David Blei and Aaron Zweig for helpful discussions and valuable feedback on the project’s core ideas. This work was supported by the NIH NHGRI grant R01HG012875, and grant number 2022-253560 from the Chan Zuckerberg Initiative DAF, an advised fund of Silicon Valley Community Foundation.
\end{acknowledgements}

% References
\bibliography{uai2026-template}

\newpage

\onecolumn

\title{Scalable Contrastive Causal Discovery under Unknown Soft Interventions\\(Supplementary Material)}
\maketitle
\appendix
\section{Background}
\subsection{$\Psi$-Markov and $\Psi$-FCI}
Learning from soft interventions with \emph{unknown} targets is formalized via the \emph{$\Psi$-Markov property}, which couples (i) standard within-regime Markov/CI constraints with (ii) \emph{cross-regime invariance} constraints whose premises depend on the symmetric difference of (unknown) intervention targets. Concretely, for a causal graph $D$ and a tuple of intervention targets $I$, $\Psi$-Markov requires:
(a) $Y \perp\!\!\!\perp Z \mid W$ in $D \Rightarrow P_i(y\mid w,z)=P_i(y\mid w)$ (within-regime CIs);
(b) $Y \perp\!\!\!\perp K \mid W\setminus W_K$ in a suitable mutilated graph (where $K=I_i\Delta I_j$) $\Rightarrow P_i(y\mid w)=P_j(y\mid w)$ (cross-regime invariances).
$\Psi$-Markov equivalence is defined at the level of \emph{pairs} $\langle D,I\rangle$ via equality of the induced set of admissible distribution tuples, and $\Psi$-FCI is shown sound and complete for the resulting $\Psi$-Markov equivalence class (i.e., maximally informative in the population limit).

\subsection{Contributions}
We focus on extending the theory to a scalable and subset-restricted pipeline under 2 regimes. Our specific theortical contributions are the following:

\subsubsection{An Estimand for Scalable Subset-restricted Learning}
$\Psi$-FCI targets the $\Psi$-equivalence class under a global oracle and outputs a $\Psi$-PAG/I-MAG object. We define a model-matched estimand: the essential graph induced by exactly the subset-level PDAGs and cross-regime invariances we can reliably estimate at scale with our model(Definitions \ref{res-psi-markov}, \ref{def:res-psi-equivalence}, \ref{def:Gtest}). The model only ever sees local PDAGs and a finite invariance query set and only estimates partial orientation that is identifiable given those evidence.

\subsubsection{Restricted-$\Psi$ Sound Contrastive Orientation Rules}
Under the definition of our estimand, we prove that the contrastive orientation rules are restricted-$\Psi$ equivalence sound and only orient edges compelled in the restricted-$\Psi$ essential graph(Propositions ~\ref{prop:ssi-psi}, \ref{prop:cvt-psi}, \ref{prop:dpt-psi}). The rules are specific to local witnesses and subset aggregation, and the soundness proof is correspondingly novel since existing alignment rules is proved by augmenting global graphs.

\subsubsection{Necessity of Contrasting Across Regimes}
To showcase the necessity of our contrastive framework, we establish strict separation between our formulation and a subset aggregator without contrastive structures: Any subset aggregator that does not use cross-regime contrasts cannot orient a non-vanishing fraction of edges that are compelled by our orientation(Lemma \ref{lem:noncontrastive-impossible}). We also demonstrate additional aggregation guarantees that ensures model soundness(Theorem \ref{thm:ctr-dom-soundness}). 

\subsubsection{Asymptotic Recovery of the Estimand}
We lastly show that we asymptotically recover our estimand which is the test-induced restricted $\Psi$ essential graph(Theorem ~\ref{thm:main-psi}). This also serves as an identifiability guarantee(under the subset restricted setting) of our model. 

%%%%%%%%%%%%%%%%%%%%%%%%%%%%%%%%%%%%%%%%%%%%%%%%%%%%%%%%%%%%%%%%%%%%%%%%%%%%
\section{Setup}
\subsection{Two-regime data and shared DAG.}
Let $V=\{1,\dots,p\}$ index observed variables $X_V$. We observe two regimes $c\in\{0,1\}$ with distributions
$P^{(c)}$ on $X_V$. Assume there exists a DAG $G=(V,E)$ such that each $P^{(c)}$ is Markov and faithful to $G$ and
the two regimes share the same graph $G$ while allowing mechanism changes across regimes at an unknown subset of
nodes.

\subsection{Operator and Notations}\label{appendix::ops}
\paragraph{Population conditional invariance.}
Let $V$ index the observed variables and let $P^{(0)}$ and $P^{(1)}$ denote the joint distributions over $X_V$
in regimes $0$ and $1$, respectively. For $v\in V$ and $Z\subseteq V\setminus\{v\}$, define the population invariance
indicator
\[
\Inv(v\mid Z)
\ :=\
\ind\!\Big\{
P^{(0)}(X_v\mid X_Z)=P^{(1)}(X_v\mid X_Z)
\Big\},
\qquad
\Chg(v\mid Z)\ :=\ 1-\Inv(v\mid Z).
\]
Here equality represents $P^{(0)}(X_v\mid X_Z=z)=P^{(1)}(X_v\mid X_Z=z)$ for almost every $z$.
The model uses estimators $\hInv(v\mid Z)$ and $\hChg(v\mid Z):=1-\hInv(v\mid Z)$.

\paragraph{Admissible subsets.}
Let $V$ index the observed variables. We define a nonempty family $\Sadm \subseteq 2^V$ of \emph{admissible subsets}.
Intuitively, $\Sadm$ specifies the collections of variables on which the model runs a single-regime causal discovery routine with controlled statistical error and feasible computational cost.
Throughout, all local graph constraints and all conditioning sets appearing in our rules are required to lie within
some $S\in\Sadm$.

\paragraph{Admissible tests and testing families.}
Let $\Sadm\subseteq 2^V$ be the collection of \emph{admissible subsets} on which we estimate local within-regime structure,
and let
\[
\Tadm \ \subseteq\ \big\{(v,Z): v\in V,\; Z\subseteq V\setminus\{v\}\big\}
\]
be the collection of tested invariance queries. We say that $(\Sadm,\Tadm)$ is an \emph{admissible testing family} if
for every $(v,Z)\in\Tadm$ there exists at least one $S\in\Sadm$ such that
\[
\{v\}\cup Z \ \subseteq\ S.
\]
Every invariance query used by the model is supported by some admissible subset containing all involved variables.

%%%%%%%%%%%%%%%%%%%%%%%%%%%%%%%%%%%%%%%%%%%%%%%%%%%%%%%%%%%%%%%%%%%%%%%%%%%%
\section{Definitions}\label{appendix::defs}

\begin{definition}[Local PDAG]\label{def::pdag}
    Assume we have per regime causal discovery estimator $\mathcal{D}^{(c)}$. If $P_S^{(c)}$ is the marginal distribution on some $S \in \Sadm$, we define local PDAG as:
    \[
    E_{S}^{(c)} = \mathcal{D}^{(c)}(P_S^{(c)})
    \]
    Notice that we do not assume $E_{S}^{(c)}$ represent the MEC of true subgraph $G[S]$ due to potential unseen confounders introduced by marginalizing with subsets. 
\end{definition}

\begin{definition}[Restricted $\Psi$-Markov and Constraint Family]\label{res-psi-markov}
Fix a DAG $G=(V,E)$ and two regimes $c\in\{0,1\}$.Consider an admissible testing family $(\Sadm,\Tadm)$ and a within-regime
discovery operator $\mathcal D^{(c)}$ as in Definition~\ref{def::pdag}. Set a two-regime fixed-graph modular soft-intervention SCM class $\mathcal M$. Let $\mathcal M(G)$ denote the set of two-regime distribution pairs generated by models in $\mathcal M$ with underlying DAG $G$. We say $(P^{(0)},P^{(1)}) \in \mathcal M(G)$ satisfy restricted
$\Psi$-Markov with respect to $G$ and $(\Sadm,\Tadm)$ if for all $S\in\Sadm$ and $c\in\{0,1\}$ the local population
PDAG is given by
\[
E_S^{(c)} \ :=\ \mathcal D^{(c)}(P_S^{(c)}),
\]
and for every tested pair $(v,Z)\in\Tadm$ the population invariance indicator is
\[
\Inv(v\mid Z)
\ :=\
\ind\!\Big\{
P^{(0)}(X_v\mid X_Z)=P^{(1)}(X_v\mid X_Z)
\Big\}.
\]

We can in turn define the restricted $\Psi$-constraint family as
\[
\mathcal F_{\psi}(\Sadm,\Tadm)
\ :=\
\big\{\, E^{(c)}_S(\cdot)\ :\ S\in\Sadm,\ c\in\{0,1\}\,\big\}
\ \cup\
\big\{\, \Inv(v\mid Z)(\cdot)\ :\ (v,Z)\in\Tadm \,\big\},
\]
where for a distribution pair $(P^{(0)},P^{(1)})$,
$E^{(c)}_S(P^{(0)},P^{(1)})$ denotes the population PDAG summary produced by applying $\mathcal D^{(c)}$ to $P^{(c)}_S$, and
$\Inv(v\mid Z)(P^{(0)},P^{(1)})$ denotes the invariance indicator defined from $(P^{(0)},P^{(1)})$.
\end{definition}

\begin{definition}[Restricted $\Psi$-equivalence]\label{def:res-psi-equivalence}
Fix an admissible testing family $(\Sadm,\Tadm)$, a within-regime discovery operator $\mathcal D$ (Definition~\ref{def::pdag}),
and a two-regime fixed-graph modular soft-intervention SCM class $\mathcal M$.
Let $\mathcal M(G)$ denote the set of two-regime distribution pairs generated by models in $\mathcal M$ with underlying DAG $G$.

Define the restricted $\Psi$-image of $G$ under $(\Sadm, \Tadm)$ as
\[
\Psi^{\mathrm{res}}_{\Sadm,\Tadm}(G;\mathcal M)
:=\Big\{\big(\{E_S^{(c)}\}_{S\in\Sadm,\ c\in\{0,1\}},\ \{\Inv(v\mid Z)\}_{(v,Z)\in\Tadm}\big)
:\ (P^{(0)},P^{(1)})\in \mathcal M(G)\Big\},
\]
where $E_S^{(c)}=\mathcal D(P_S^{(c)})$ and $\Inv(v\mid Z)$ is defined from $(P^{(0)},P^{(1)})$ as in
Definition~\ref{res-psi-markov}.

Two DAGs $G$ and $G'$ on $V$ are \emph{restricted $\Psi$-Markov equivalent} (w.r.t.\ $(\Sadm,\Tadm)$, $\mathcal D$, and $\mathcal M$) if
\[
\Psi^{\mathrm{res}}_{\Sadm,\Tadm}(G;\mathcal M)
=\Psi^{\mathrm{res}}_{\Sadm,\Tadm}(G';\mathcal M).
\]
\end{definition}

\begin{definition}[Test-induced $\Psi$ essential graph]\label{def:Gtest}
Fix an admissible testing family $(\Sadm,\Tadm)$, a within-regime discovery operator $\mathcal D$, and a two-regime SCM class $\mathcal M$.
For a DAG $G$ on $V$, define its restricted-$\Psi$ equivalence class as
\[
[G]_{\mathrm{res}}
\ :=\
\big\{\, G' \text{ DAG on } V : \Psi^{\mathrm{res}}_{\Sadm,\Tadm}(G';\mathcal M)=\Psi^{\mathrm{res}}_{\Sadm,\Tadm}(G;\mathcal M)\,\big\}.
\]

Denote $e(.)$ as edge between two nodes. Define the test-induced $\Psi$ essential graph $\Gtest(G)$ as the PDAG on $V$ whose skeleton is the
intersection of graph skeletons in the equivalence class $[G]_{\mathrm{res}}$, i.e.
\[
e(u - v) \in \Gtest(G)
\quad\Longleftrightarrow\quad
e(u - v) \in G', \quad \forall G' \in [G]_{\mathrm{res}},
\]
and whose orientations are also the intersections, i.e.:
\[
e(u \to v) \in \Gtest(G)
\quad\Longleftrightarrow\quad
e(u \to v) \in G' \quad \forall G'\in [G]_{\mathrm{res}}.
\]
All remaining edges in the skeleton are left undirected. Optionally, apply a closure operator 
to obtain the maximally oriented PDAG.
\end{definition}

\begin{definition}[Admissible witness family]\label{def:witness}
Fix an admissible testing family $(\Sadm,\Tadm)$ and the population local PDAGs
$\{E_S^{(c)}: S\in\Sadm,\ c\in\{0,1\}\}$, where $E_S^{(c)}=\mathcal D(P_S^{(c)})$.
For a PDAG $H$ on node set $S$ and a node $v\in S$, let $\De^{\mathrm{def}}_{H}(v)$ denote the set of
\emph{definite descendants} of $v$ in $H$, i.e.,
\[
\De^{\mathrm{def}}_{H}(v)
:=\bigcap_{D\in\Ext{H}} \De_{D}(v),
\]
where $\Ext{H} \ne \emptyset$ is the set of all DAG extensions of $H$ and $\De_D(v)$ is the descendant set of $v$ in $D$.

For each $v\in V$, define the admissible witness family as
\[
\mathcal W(v)
:=\Big\{Z:\ (v,Z)\in\Tadm,\ \exists S\in\Sadm\ \text{s.t.}\ \{v\}\cup Z\subseteq S,\ 
Z\cap\De^{\mathrm{def}}_{E_S^{(0)}}(v)=\emptyset,\ 
Z\cap\De^{\mathrm{def}}_{E_S^{(1)}}(v)=\emptyset
\Big\}.
\]
In other words, $Z$ does not contain any definite descendants of $v$ in the directed part of PDAGs.
\end{definition}

%%%%%%%%%%%%%%%%%%%%%%%%%%%%%%%%%%%%%%%%%%%%%%%%%%%%%%%%%%%%%%%%%%%%%%%%%%%%
\section{Assumptions}\label{appendix::assmps}
\begin{assumption}[Finite admissible scope]\label{ass:finite-scope}
The families $(\Sadm,\Tadm)$ are finite (or of size permitting uniform consistency). Each $(v,Z)\in\Tadm$ satisfies
$\{v\}\cup Z\subseteq S$ for some $S\in\Sadm$.
\end{assumption}

\begin{assumption}[Local PDAG consistency]\label{ass:local-pdag-cons}
Fix a within-regime discovery operator $\mathcal D$ (Definition~\ref{def::pdag}). For each $S\in\Sadm$ and
$c\in\{0,1\}$, let the population local PDAG be $E_S^{(c)}:=\mathcal D(P_S^{(c)})$ and let
$\widehat E_S^{(c)}:=\mathcal D(\widehat P_S^{(c)})$ denote its sample-based estimate. We assume
\[
\Prb\!\Big(\forall S\in\Sadm,\forall c\in\{0,1\}: \widehat E_S^{(c)} = E_S^{(c)}\Big)\to 1
\qquad \text{as } n\to\infty.
\]
where $n$ is number of samples.
\end{assumption}

\begin{assumption}[Invariance-test consistency]\label{ass:inv-cons}
For each $(v,Z)\in\Tadm$, the invariance test satisfies
\[
\Prb\!\Big(\forall (v,Z)\in\Tadm: \hInv(v\mid Z)=\Inv(v\mid Z)\Big)\to 1
\qquad \text{as } n\to\infty.
\]
where $\hInv(v\mid Z)$ is the sample-based estimate of the invariance test and $n$ is number of samples.
\end{assumption}

\begin{assumption}[Off-target invariance and nontrivial detectability]\label{ass:oti-detect}
Assume $(P^{(0)},P^{(1)}) \in \mathcal{M}(G)$ where $\mathcal{M}(G)$ is defined in Definition \ref{def:res-psi-equivalence},
with intervention targets $I\subseteq V$ in regime $c=1$. Then:
\begin{itemize}
    \item For every $v \notin I$,
    \[
    P^{(0)}(X_v \mid X_{\PA_G(v)}) = P^{(1)}(X_v \mid X_{\PA_G(v)}).
    \]
    \item There exists $v \in I$ and some $Z_v \in \mathcal{W}(v)$ such that
    \[
    \Chg(v\mid Z_v) = 1.
    \]
    we also use this definition to call the intervention target node $v$ detectable. 
\end{itemize}
\end{assumption}

\begin{assumption}[Regime faithfulness on tested witnesses]\label{ass:change-faithfulness}
Let $(P^{(0)},P^{(1)})\in\mathcal M(G)$ with (unknown) intervention target set $I\subseteq V$.
Introduce a binary regime label $C\in\{0,1\}$ with $\pi_0,\pi_1>0$, $\pi_0+\pi_1=1$, and define
\[
P^*(C=c,X)\ :=\ \pi_c\,P^{(c)}(X).
\]
Let the augmented graph be
\[
G^* \ :=\ (V\cup\{C\},\ E\cup\{C\to t:\ t\in I\}).
\]
Then for every admissible subset $S\in\Sadm$ and every admissible tested witness query $(v,Z)\in\Tadm$ satisfying
$Z\in\mathcal W(v)$, we have
\[
C \perp_d v \mid Z \ \text{in}\ G^*
\quad\Longleftrightarrow\quad
C \indep X_v \mid X_Z \ \text{under}\ P^* .
\]
\end{assumption}

Under 2 regimes, we propose the following lemma:

\begin{lemma}[Two-regime conditional invariance as $C$-conditional independence]\label{lem:inv-ci}
Let $P^{(0)}$ and $P^{(1)}$ be two distributions over $X$, and assume $\pi_0,\pi_1>0$
with $\pi_0+\pi_1=1$. Introduce a binary regime label $C\in\{0,1\}$ and define the joint distribution
\[
P^*(C=c,X)\ :=\ \pi_c\,P^{(c)}(X).
\]
Then for any node $v$ and conditioning set $Z$ where $\Inv(v\mid Z):=\ind\{P^{(0)}(X_v\mid X_Z)=P^{(1)}(X_v\mid X_Z)\}$, we have
\[
\Inv(v\mid Z)=1 \iff C\indep X_v\mid X_Z,
\qquad
\Inv(v\mid Z)=0 \iff C \notindep X_v\mid X_Z,
\]
all under $P^*$.
\end{lemma}

\begin{proof}[Proof]
Write the mixture as $P^*(C=c,X)=\pi_cP^{(c)}(X)$ with $\pi_c>0$. Then for $c\in\{0,1\}$,
\begin{align*}
P^*(X_v\mid X_Z, C=c)
&=\frac{P^*(X_v,X_Z\mid C=c)}{P^*(X_Z\mid C=c)}
=\frac{\pi_c P^{(c)}(X_v,X_Z)/\pi_c}{\pi_c P^{(c)}(X_Z)/\pi_c}
= P^{(c)}(X_v\mid X_Z).
\end{align*}
Hence
\[
P^{(0)}(X_v\mid X_Z)=P^{(1)}(X_v\mid X_Z)
\quad\Longleftrightarrow\quad
P^*(X_v\mid X_Z,C=0)=P^*(X_v\mid X_Z,C=1).
\]
By definition of conditional independence,
\[
C\indep X_v\mid X_Z
\quad\Longleftrightarrow\quad
P^*(X_v\mid X_Z,C=0)=P^*(X_v\mid X_Z,C=1),
\]
Since $\Inv(v\mid Z):=\ind\{P^{(0)}(X_v\mid X_Z)=P^{(1)}(X_v\mid X_Z)\}$, we have
\[
\Inv(v\mid Z)=1 \iff C\indep X_v\mid X_Z
\]
The inverse relationship for $\Inv(v \mid z) = 0$ holds symmetrically. 
\end{proof}

\begin{remark}[Handling example failed cases]
By Lemma~\ref{lem:inv-ci}, $\Inv(v\mid Z)=1$ is equivalent to $C\indep X_v\mid X_Z$ under the two-regime mixture $P^*$,
and $\Chg(v\mid Z)=1$ to $C\notindep X_v\mid X_Z$. Assumption~\ref{ass:change-faithfulness} enforces restricted
faithfulness of these tested $C$-independence to d-separation in the augmented graph $G^*$. Thus it rules out
(i) self-cancellation (d-connection but $C\indep X_v\mid X_Z$) and (ii) Bayes reversal on admissible witnesses
(d-separation but $C\notindep X_v\mid X_Z$).
\end{remark}

\begin{assumption}[Closure soundness]\label{ass:closure-sound}
Let $\Cl(\cdot)$ denote the closure operator.
If a set of directed constraints is consistent with $\Gtest$, then $\Cl(\cdot)$ does not introduce orientations that
contradict $\Gtest$.
\end{assumption}

%%%%%%%%%%%%%%%%%%%%%%%%%%%%%%%%%%%%%%%%%%%%%%%%%%%%%%%%%%%%%%%%%%%%%%%%%%%%
\section{Restricted $\Psi$-Sound Orientation Rules}
\subsection{Single-Sided Invariance}
\begin{proposition}[Single-Sided Invariance (SSI)]\label{prop:ssi-psi}
Let $S\in\Sadm$ with $\{i,j\}\subseteq S$ such that $i-j$ is an undirected adjacency in both population local PDAGs
$E_S^{(0)}$ and $E_S^{(1)}$ (Definition~\ref{def::pdag}). Let $Z:=S\setminus\{i,j\}$ such that
$(j, Z), (i, Z) \in \Tadm$ and $Z \in \mathcal{W}(j)$, $Z \in \mathcal{W}(i)$.
Assume moreover that $i$ and $j$ are adjacent in every $G'\in [G]_{\mathrm{res}}$. If
\[
\Chg(j\mid Z)=1,
\qquad
\Inv(i\mid Z)=1,
\]
then $i\to j$ is invariant across the restricted $\Psi$-equivalence class and thus is oriented identically in $\Gtest(G)$.
\end{proposition}

\begin{proof}[Proof]
\textbf{Restricted $\Psi$-equivalence soundness.}
Fix a DAG $G=(V,E)$ such that $(P^{(0)},P^{(1)})\in \mathcal M(G)$ and consider the associated two-regime mixture
$P^*$ and augmented graph $G^*$ from Assumption~\ref{ass:change-faithfulness}. Assume for contradiction that $j\to i$ in $G[S]$.
By Lemma~\ref{lem:inv-ci},
\[
\Chg(j\mid Z)=1 \ \Longleftrightarrow\ C \not\!\indep X_j \mid X_{Z},
\qquad
\Inv(i\mid Z)=1 \ \Longleftrightarrow\ C \indep X_i \mid X_{Z},
\]
all under $P^*$. Since $(j,Z)\in\Tadm$ and $Z\in\mathcal W(j)$, Assumption~\ref{ass:change-faithfulness} implies
\[
C \not\!\indep X_j\mid X_Z \ \Longrightarrow\ C \not\!\perp_d j \mid Z \ \text{in}\ G^* .
\]
Thus there exists a d-connecting path $\pi$ given $Z$ from $C$ to $j$ in $G^*$. Concatenate $\pi$ with the directed edge $j\to i$. Because $j\notin Z$, the concatenated path is unblocked given $Z$ from $C$ to $i$,
so $C\not\!\perp_d i\mid Z$ in $G^*$. Additionally, $j$ cannot be a collider since we concatenated $j\to i$. Since $(i,Z)\in\Tadm$ and $Z\in\mathcal W(i)$, applying
Assumption~\ref{ass:change-faithfulness} again yields
\[
C \not\!\perp_d i \mid Z \ \Longrightarrow\  C \not\!\indep X_i \mid X_{Z} \ \text{under } P^*,
\]
i.e.\ $\Chg(i\mid Z)=1$, contradicting $\Inv(i\mid Z)=1$. Therefore, $j\to i$ cannot hold in $G[S]$.\\\\
\textbf{Compelled orientation in the test-induced $\Psi$ essential graph.}
Let $G'\in [G]_{\mathrm{res}}$ be arbitrary. By Definition~\ref{def:res-psi-equivalence},
\[
\Psi^{\mathrm{res}}_{\Sadm,\Tadm}(G';\mathcal M)=\Psi^{\mathrm{res}}_{\Sadm,\Tadm}(G;\mathcal M),
\]
so the same restricted $\Psi$-image tuple witnessing the premises of the rule is also achievable under some
$(Q^{(0)},Q^{(1)})\in\mathcal M(G')$. Repeating the contradiction argument above for $G'$ with its corresponding
mixture $Q^*$ and augmented graph $(G')^*$ rules out $j\to i$ in $G'[S]$. Assume $i$ and $j$
are adjacent in every $G'\in [G]_{\mathrm{res}}$, it follows that $i\to j$ holds in every $G'\in [G]_{\mathrm{res}}$,
i.e.\ $i\to j$ is compelled across the restricted $\Psi$-equivalence class. By Definition~\ref{def:Gtest}, this implies
$i\to j$ appears in $\Gtest(G)$.
\end{proof}

\subsection{Contrastive V structure}

\begin{proposition}[Contrastive V-structure(CVT)]\label{prop:cvt-psi}
Let $S\in\Sadm$ with $\{i,j,k\}\subseteq S$ such that $i-j-k$ is an unshielded triple in the local skeleton on $S$.
Let $Z:=S\setminus\{i,j,k\}$ and assume $(i,Z),(j,Z),(k,Z)\in\Tadm$ with $Z\in\mathcal W(i)\cap\mathcal W(j)\cap\mathcal W(k)$.
Assume moreover that $i-j$ and $j-k$ are edges in the common skeleton of $[G]_{\mathrm{res}}$ (equivalently, in the skeleton of $\Gtest(G)$).
If
\[
\Chg(j\mid Z)=1,\qquad \Inv(i\mid Z)=1,\qquad \Inv(k\mid Z)=1,
\]
then $i\to j\leftarrow k$ is compelled across $[G]_{\mathrm{res}}$ and hence appears in $\Gtest(G)$.
\end{proposition}

\begin{proof}[Proof]
\textbf{Restricted $\Psi$-equivalence soundness.}
Fix a DAG $G=(V,E)$ such that $(P^{(0)},P^{(1)})\in\mathcal M(G)$ and consider the associated two-regime mixture $P^*$
and augmented graph $G^*$ from Assumption~\ref{ass:change-faithfulness}. Let $Z:=S\setminus\{i,j,k\}$.
By Lemma~\ref{lem:inv-ci}, under $P^*$ we have
\[
\Chg(j\mid Z)=1 \Longleftrightarrow C \not\!\indep X_j\mid X_Z,\qquad
\Inv(i\mid Z)=1 \Longleftrightarrow C \indep X_i\mid X_Z,\qquad
\Inv(k\mid Z)=1 \Longleftrightarrow C \indep X_k\mid X_Z .
\]
Since $(j,Z)\in\Tadm$ and $Z\in\mathcal W(j)$, Assumption~\ref{ass:change-faithfulness} implies
\[
C \not\!\indep X_j\mid X_Z \ \Longrightarrow\ C \not\!\perp_d j\mid Z \quad\text{in }G^*,
\]
so there exists a path $\pi$ from $C$ to $j$ in $G^*$ that is unblocked given $Z$.

Assume for contradiction that $j\to i$ in $G[S]$. Concatenating $\pi$ with the directed edge $j\to i$ yields a path
from $C$ to $i$ that is unblocked given $Z$. Because $j\notin Z$ and the added edge is $j \to i$, the node $j$ is not a collider on the
concatenated path. Therefore $C\not\!\perp_d i\mid Z$ in $G^*$.
Since $(i,Z)\in\Tadm$ and $Z\in\mathcal W(i)$, Assumption~\ref{ass:change-faithfulness} gives
\[
C \not\!\perp_d i\mid Z \ \Longrightarrow\ C \not\!\indep X_i\mid X_Z \quad\text{under }P^*,
\]
contradicting $\Inv(i\mid Z)=1$. Hence $j\to i$ is impossible, and since $i$ and $j$ are adjacent in $G[S]$, we must
have $i\to j$ in $G[S]$. Invoking the same argument for $\Inv(k\mid Z)=1$ similarly rules out $j\to k$ and yields $k\to j$ in $G[S]$.
Because $i$ and $k$ are non-adjacent, this implies $i\to j\leftarrow k$ is in $G[S]$.

\textbf{Compelled orientation in the test-induced $\Psi$ essential graph.}
Let $G'\in [G]_{\mathrm{res}}$ be arbitrary. By Definition~\ref{def:res-psi-equivalence},
\[
\Psi^{\mathrm{res}}_{\Sadm,\Tadm}(G';\mathcal M)=\Psi^{\mathrm{res}}_{\Sadm,\Tadm}(G;\mathcal M),
\]
so the same restricted $\Psi$-image tuple witnessing the premises of the proposition is achievable under some
$(Q^{(0)},Q^{(1)})\in\mathcal M(G')$. Repeating the contradiction arguments above for $G'$ (with its corresponding
mixture $Q^*$ and augmented graph $(G')^*$) rules out $j\to i$ and $j\to k$ in $G'[S]$, hence forces
$i\to j$ and $k\to j$. Since the unshielded triple $i-j-k$ exists in every $G'\in [G]_{\mathrm{res}}$ by assumption,
$i$ and $k$ remain non-adjacent in every such $G'$, so $i\to j\leftarrow k$ holds in every $G'\in [G]_{\mathrm{res}}$.
Therefore the collider is compelled across $[G]_{\mathrm{res}}$, and by Definition~\ref{def:Gtest} it appears in
$\Gtest(G)$.
\end{proof}

\subsection{Contrastive Discriminating Path}

We first provide some definitions that will be useful for the proofs.
\begin{definition}[Candidate discriminating path in a local skeleton]\label{def:candidate-dp-motif}
Let $S\in\Sadm$ and let $H$ be a partially directed graph on vertex set $S$.
A path
\[
\pi=\langle v_0,v_1,\ldots,v_{m-1},j,k\rangle
\]
in the skeleton of $H$ is a \emph{candidate discriminating-path motif for $(j,k)$} if:
\begin{enumerate}
    \item[(i)] $\pi$ is simple (no repeated vertices);
    \item[(ii)] $v_0$ is non-adjacent to $k$ in the skeleton of $H$;
    \item[(iii)] $j$ is adjacent to $k$ in the skeleton of $H$.
\end{enumerate}
The path is called \emph{common} if the same skeleton path exists in both local skeletons.
\end{definition}

\begin{definition}[Replay-node local status and replay assignment]\label{def:replay-status}
Fix a candidate discriminating-path 
\[
\pi=\langle v_0,\ldots,v_{m-1},j,k\rangle .
\]
For each internal node $v_\ell$ ($1\le \ell \le m-1$), define its \emph{local discriminating-path status}
(with respect to $(\pi,k)$) as the tuple
\[
\mathrm{st}_\pi(v_\ell)
\;:=\;
\Big(
\mathbf{1}\{\text{$v_\ell$ is a collider on $\pi$}\},
\mathbf{1}\{\text{$v_\ell$ is a parent of $k$}\}
\Big)\in\{0,1\}^2,
\]

Let $\mathcal V_{\mathrm{rep}}(\pi)\subseteq \{v_1,\ldots,v_{m-1}\}$ be a designated set of internal nodes.
A \emph{replay assignment} on $\mathcal V_{\mathrm{rep}}(\pi)$ is a map
\[
\sigma:\mathcal V_{\mathrm{rep}}(\pi)\to\{0,1\}^2,
\qquad
v\mapsto \sigma(v),
\]
that specifies the local discriminating-path statuses of the replay nodes.
\end{definition}

\begin{definition}[Replay-node set]\label{def:replay-nodes}
Fix a candidate discriminating-path motif $\pi$ and the ambient local graph constraints (skeleton adjacencies and any
already-known endpoint marks) on $S$.

A set $\mathcal V_{\mathrm{rep}}(\pi)\subseteq \{v_1,\ldots,v_{m-1}\}$ is called a \emph{sufficient replay-node set}
if, once a replay assignment $\sigma$ on $\mathcal V_{\mathrm{rep}}(\pi)$ is fixed (and is compatible with the ambient
local constraints), the standard discriminating-path logic determines a unique mark on the edge $(j,k)$.

It is called a \emph{minimal replay-node set} if it is sufficient and no proper subset of it is sufficient.
\end{definition}

\begin{definition}[Discriminating-path mark]\label{def:dp-mark}
Fix a candidate discriminating path 
\[
\pi=\langle v_0,\ldots,v_{m-1},j,k\rangle
\]
and a sufficient replay-node set $\mathcal V_{\mathrm{rep}}(\pi)$.

Define the \emph{discriminating path mark}
\[
\mathrm{DP}_\pi:\ \sigma \mapsto \mathrm{DP}_\pi(\sigma)
\]
on replay assignments $\sigma$ that are compatible with the local constraints, where $\mathrm{DP}_\pi(\sigma)$ is the
mark returned by the standard discriminating-path logic for the target edge $(j,k)$ after substituting the replay
statuses specified by $\sigma$.

$\mathrm{DP}_\pi(\sigma)$ may be represented as:
\begin{itemize}
    \item[(a)] an endpoint \emph{mark} on $(j,k)$ (e.g., tail vs.\ arrowhead at $j$), or
    \item[(b)] a full \emph{orientation} of $(j,k)$ when the local logic (possibly together with closure rules) determines it.
\end{itemize}
We refer to either representation generically as the \emph{discriminating-path mark} on $(j,k)$.
\end{definition}

Here we formulate and prove the assisting lemmas and the proposition.
\begin{lemma}[Concatenation of unblocked paths at a non-collider]\label{lem:active-concat}
Let $\pi_1$ be an unblocked path from $a$ to $b$ and let $\pi_2$ be an unblocked path $b$ to $d$ in a DAG, both conditioning on $Z$.
Assume:
\begin{enumerate}
    \item[(i)] $\pi_1$ and $\pi_2$ intersect only at $b$;
    \item[(ii)] $b\notin Z$;
    \item[(iii)] the two incident edge-marks at $b$ in the concatenated trail $\pi_1\oplus \pi_2$ do not make $b$ a collider.
\end{enumerate}
Then the concatenated path $\pi_1\oplus \pi_2$ is unblocked from $a$ to $d$ given $Z$.
\end{lemma}

\begin{proof}[Proof]
Because $\pi_1$ and $\pi_2$ are each unblocked given $Z$, every internal node of $\pi_1$ and of $\pi_2$ satisfies the
usual $d$-separation condition. Under (i), concatenation introduces no repeated internal nodes; the only new
internal node in $\pi_1\oplus\pi_2$ is the join node $b$. By (iii), $b$ is a non-collider in the concatenated path, and by
(ii), $b\notin Z$, so $b$ is also unblocked given $Z$. Therefore every internal node of $\pi_1\oplus\pi_2$ is unblocked
given $Z$, and $\pi_1\oplus\pi_2$ is unblocked from $a$ to $d$ given $Z$.
\end{proof}

\begin{lemma}[First replay disagreement under alternative discriminating path mark]\label{lem:dpt-first-mismatch}
Consider a candidate discriminating path
\[
\pi=\langle v_0,\ldots,v_{m-1},j,k\rangle
\]
and replay nodes $\mathcal V_{\mathrm{rep}}(\pi)\subseteq\{v_1,\ldots,v_{m-1}\}$.
Let $\sigma^\star$ be the replay-node status assignment that yields the target discriminating-path mark
$\mathrm{DP}(\pi)$ on $(j,k)$.

If an alternative replay assignment $\sigma\neq \sigma^\star$ changes the discriminating-path mark on $(j,k)$,
then there exists a first replay node $r\in\mathcal V_{\mathrm{rep}}(\pi)$ (ordered from $j$ to $v_0$)
such that $\sigma(r)\neq \sigma^\star(r)$ while $\sigma(u)=\sigma^\star(u)$ for all replay nodes $u$ between $j$ and $r$
along $\pi$.
\end{lemma}

\begin{proof}[Proof]
Let the replay nodes be listed in the order from $j$ to $v_0$:
\[
\mathcal V_{\mathrm{rep}}(\pi)=\{r_1,\ldots,r_q\},
\]
where $r_1$ is the replay node closest to $j$ and $r_q$ is the one farthest from $j$.

Assume $\sigma$ is an alternative replay assignment whose induced discriminating-path mark on $(j,k)$ differs from
the target mark induced by $\sigma^\star$. If $\sigma(r_a)=\sigma^\star(r_a)$ for every $a\in\{1,\ldots,q\}$, then
$\sigma$ and $\sigma^\star$ present the same replay information to the discriminating-path logic. Hence they must induce the same discriminating-path mark on $(j,k)$, contradicting the assumption that the marks differ.

Therefore, the set
\[
A:=\{a\in\{1,\ldots,q\}:\sigma(r_a)\neq \sigma^\star(r_a)\}
\]
is a nonempty finite set. Take the minimum element $a^\star:=\min A$.
Set $r:=r_{a^\star}$. By construction,
\[
\sigma(r)\neq \sigma^\star(r),
\]
and for every replay node $u=r_b$ with $b < a^\star$,
we have
\[
\sigma(u)=\sigma^\star(u).
\]
Thus $r$ is the first replay disagreement.
\end{proof}

\begin{lemma}[Valid path from trigger node to first mismatch node]\label{lem:dpt-j-to-r}
Fix a candidate discriminating path
\[
\pi=\langle v_0,\ldots,v_{m-1},j,k\rangle
\]
and a target replay assignment $\sigma^\star$ on $\mathcal V_{\mathrm{rep}}(\pi)\subseteq\{v_1,\ldots,v_{m-1}\}$.
Let $Z^\pi$ be a tested witness set such that $(j,Z^\pi)\in\Tadm$ and $(v,Z^\pi)\in\Tadm$ for every
$v\in\mathcal V_{\mathrm{rep}}(\pi)$, with $Z^\pi\in\mathcal W(j)\cap \bigcap_{v\in\mathcal V_{\mathrm{rep}}(\pi)}\mathcal W(v)$.

Assume at a trigger node $j$, that $\Chg(j\mid Z^\pi)=1$. Let $\sigma$ be any alternative replay assignment whose discriminating path mark on $(j,k)$ differs from the target mark, and let $r$ be the first replay disagreement given by Lemma~\ref{lem:dpt-first-mismatch}.

Assume the following hold for this $(\pi,\sigma^\star,Z^\pi)$:
\begin{enumerate}
    \item[(a)] The segment of $\pi$ from $r$ to $j$ is present in the common skeleton.
    \item[(b)] Under the marks induced by $\sigma$ on that segment, every internal node of the segment is unblocked conditioning on $Z^\pi$.
    \item[(c)] The edge of the segment next to $j$ has a tail at $j$, and $j\notin Z^\pi$.
    \item[(d)] The unblocked path from $C$ to $j$ given $Z^\pi$ intersects the segment from $r$ to $j$ only at $j$.
\end{enumerate}
Then
\[
C\not\!\perp_d r\mid Z^\pi
\qquad \text{in } G^* .
\]
\end{lemma}

\begin{proof}[Proof]
By Lemma~\ref{lem:inv-ci} and Assumption~\ref{ass:change-faithfulness}, $\Chg(j\mid Z^\pi)=1$ implies
$C\not\!\perp_d j\mid Z^\pi$ in $G^*$. Hence there exists an unblocked path $\pi_1$ from $C$ to $j$ given $Z^\pi$.

By (a) and (b), the segment of $\pi$ from $j$ to $r$ forms an unblocked path $\pi_2$ conditioning on $Z^\pi$. By (c), the join node $j$ is a non-collider in the concatenation and $j\notin Z^\pi$.
Applying Lemma~\ref{lem:active-concat} to $\pi_1$ and $\pi_2$ yields an unblocked path from $C$ to $r$ given $Z^\pi$.

Therefore $C\not\!\perp_d r\mid Z^\pi$ in $G^*$.
\end{proof}

\begin{lemma}[Contradiction for alternative replay assignment]\label{lem:dpt-contradiction}
Fix $S\in\Sadm$, a candidate discriminating path
\[
\pi=\langle v_0,\ldots,v_{m-1},j,k\rangle
\]
on $S$, and a target replay assignment $\sigma^\star$ on $\mathcal V_{\mathrm{rep}}(\pi)\subseteq\{v_1,\ldots,v_{m-1}\}$.
Let $Z^\pi$ be a common tested witness set supported on $S$ such that
\[
(j,Z^\pi)\in\Tadm,\qquad (v,Z^\pi)\in\Tadm\ \ \forall v\in\mathcal V_{\mathrm{rep}}(\pi),
\]
with
\[
Z^\pi\in\mathcal W(j)\cap \bigcap_{v\in\mathcal V_{\mathrm{rep}}(\pi)}\mathcal W(v).
\]
Assume
\[
\Inv(v\mid Z^\pi)=1\ \ \forall v\in\mathcal V_{\mathrm{rep}}(\pi),
\qquad
\Chg(j\mid Z^\pi)=1.
\]

Assume conditions in Lemma~\ref{lem:dpt-j-to-r} hold for every alternative replay assignment $\sigma$ whose discriminating path mark on $(j,k)$ differs from the target mark. Then no such alternative replay assignment exists. In other words, $\sigma^*$ on $(j,k)$ gives the
unique mark to $(j, k)$ compatible with the tested contrastive literals.
\end{lemma}

\begin{proof}[Proof]
Assume for contradiction that an alternative replay assignment $\sigma \neq \sigma^*$ induces a different mark on $(j,k)$.
By Lemma~\ref{lem:dpt-first-mismatch}, there exists a first replay disagreement $r\in\mathcal V_{\mathrm{rep}}(\pi)$.

By Lemma~\ref{lem:dpt-j-to-r}, $\Chg(j\mid Z^\pi)=1$ implies
\[
C\not\!\perp_d r\mid Z^\pi
\qquad \text{in } G^* .
\]
Since $(r,Z^\pi)\in\Tadm$ and $Z^\pi\in\mathcal W(r)$, Assumption~\ref{ass:change-faithfulness} gives
\[
C\not\!\perp_d r\mid Z^\pi \ \Longrightarrow\ C\not\!\indep X_r\mid X_{Z^\pi}.
\]
By Lemma~\ref{lem:inv-ci}, this implies $\Chg(r\mid Z^\pi)=1$, contradicting $\Inv(r\mid Z^\pi)=1$. Therefore $\sigma$ does not exist.
\end{proof}

\begin{proposition}[Contrastive Discriminating Path (DPT)]\label{prop:dpt-psi}
Let $S\in\Sadm$ and let $(j,k)$ be an adjacency that is undirected in both local PDAGs
$E_S^{(0)}$ and $E_S^{(1)}$. Suppose the common local skeleton on $S$ contains a candidate discriminating-path
\[
\pi=\langle v_0,\ldots,v_{m-1},j,k\rangle
\]
(with $v_0$ non-adjacent to $k$) such that the unknown replay-node statuses in $\mathcal V_{\mathrm{rep}}(\pi)\subseteq\{v_1,\ldots,v_{m-1}\}$ prevents the discriminating path mark on $(j, k)$. Assume there exists a common tested witness $Z^\pi$ supported on $S$ such that
\[
(j,Z^\pi)\in\Tadm,\qquad (v,Z^\pi)\in\Tadm\ \ \forall v\in\mathcal V_{\mathrm{rep}}(\pi),
\]
with
\[
Z^\pi\in\mathcal W(j)\cap \bigcap_{v\in\mathcal V_{\mathrm{rep}}(\pi)}\mathcal W(v),
\]
and
\[
\Inv(v\mid Z^\pi)=1\ \ \forall v\in\mathcal V_{\mathrm{rep}}(\pi),
\qquad
\Chg(j\mid Z^\pi)=1.
\]
Assume the following holds:
\begin{enumerate}
    \item[(i)] Common skeleton for $\pi$ exist in every $G'\in [G]_{\mathrm{res}}$;
    \item[(ii)] For every $G' \in [G]_{\mathrm{res}}$ and every alternative replay assignment that would change the
    discriminating-path mark on $(j,k)$ in $G'[S]$, the conditions in Lemma~\ref{lem:dpt-j-to-r} hold.
\end{enumerate}
Then the discriminating-path mark on $(j,k)$ is invariant across $[G]_{\mathrm{res}}$ and hence appears in $\Gtest(G)$.
\end{proposition}

\begin{proof}[Proof]
\textbf{Restricted $\Psi$-equivalence soundness.}
Fix $G$ with $(P^{(0)},P^{(1)})\in\mathcal M(G)$. By Lemma~\ref{lem:dpt-contradiction}, $\{\Inv(v\mid Z^\pi)\}_{v\in\mathcal V_{\mathrm{rep}}(\pi)}$ together with $\Chg(j\mid Z^\pi)=1$ uniquely determine the discriminating-path mark on $(j,k)$ among replay assignments compatible with the path skeleton. Hence the target mark must hold in $G[S]$.

\textbf{Compelled in $\Gtest$.}
Let $G'\in [G]_{\mathrm{res}}$ be arbitrary. By Definition~\ref{def:res-psi-equivalence}, the same restricted $\Psi$-image
tuple witnessing the DPT premises is achievable under some $(Q^{(0)},Q^{(1)})\in\mathcal M(G')$. Re-applying
Lemma~\ref{lem:dpt-contradiction} to $G'$ yields the same discriminating-path mark on $(j,k)$ in $G'[S]$. By assumption (i),
the motif adjacencies persist in every $G'\in [G]_{\mathrm{res}}$, so the mark holds in every such $G'$. Therefore it
is compelled across $[G]_{\mathrm{res}}$ and appears in $\Gtest(G)$ by Definition~\ref{def:Gtest}.
\end{proof}

\section{Global Restricted $\Psi$ I-EG Estimation}
\subsection{Per-regime Estimand}
Since non-contrastive aggregators do not have access to cross-regime contrast, we define a reduction of the restricted $\Psi$-image and the test-induced $\Psi$ essential graph based on per-regime PDAG information, excluding these cross-regime tested invariance queries. 

\begin{definition}[Reduced restricted $\Psi$-equivalence]\label{def:red-res-psi-equivalence}
For admissible subsets in $\mathcal{S}_\mathrm{a}$, a within-regime discovery operator $\mathcal{D}$ (Definition~\ref{def::pdag}), and a two-regime fixed-graph modular soft-intervention SCM class $\mathcal{M}$, where $\mathcal{M}(G)$ denotes the set of two-regime distribution pairs generated by models in $\mathcal{M}$ with underlying DAG $G$.

Define the reduced $\Psi$-image of $G$ under $\Sadm$ as
\[
\Psi_{\Sadm}^\mathrm{res}(G;\mathcal{M})
:=
\left\{
\left(\{E_S^{(0)}\}_{S\in\Sadm},\{E_S^{(1)}\}_{S\in\Sadm}\right)
:\ (P^{(0)},P^{(1)})\in \mathcal{M}(G),\ 
E_S^{(c)}:=\mathcal{D}(P_S^{(c)})
\right\}.
\]
Two DAGs $G,G'$ on $V$ are reduced restricted $\Psi$-Markov equivalent (w.r.t.\ $\Sadm,\mathcal{D}$ and $\mathcal{M}$) if
\[
\Psi_{\Sadm}^\mathrm{res}(G;\mathcal{M})=\Psi_{\Sadm}^\mathrm{res}(G';\mathcal{M}).
\]
\end{definition}

\begin{definition}[Reduced test-induced $\Psi$ essential graph]\label{def:red-G-test}
For admissible subsets in $\mathcal{S}_\mathrm{a}$, a within-regime discovery operator $\mathcal{D}$, and a two-regime SCM class $\mathcal{M}$. For a DAG $G$ on $V$, define its reduced restricted-$\Psi$ equivalence class as
\[
[G]_{\Sadm}^\mathrm{res}
:=
\{\, G' \text{ DAG on } V \;:\; \Psi_{\Sadm}^\mathrm{res}(G';\mathcal{M})=\Psi_{\Sadm}^\mathrm{res}(G;\mathcal{M}) \,\}.
\]

Define the reduced test-induced $\Psi$ essential graph $\redGtest(G)$ as the PDAG on $V$ whose skeleton is the common skeleton of the class $[G]_{\Sadm}^\mathrm{res}$, i.e.
\[
u - v \text{ is an edge in } \redGtest(G)
\Longleftrightarrow
u \text{ and } v \text{ are adjacent in every } G' \in [G]_{\Sadm}^\mathrm{res},
\]
and whose orientations are exactly the compelled ones:
\[
u \to v \text{ in } \redGtest(G)
\Longleftrightarrow
u \to v \text{ in every } G' \in [G]_{\Sadm}^\mathrm{res}.
\]
All remaining edges in the skeleton are left undirected.
\end{definition}

\subsection{Knowledge sets and aggregator}
Let $\{ E_S^{(c)}\}_{S \in \Sadm, c \in \{ 0, 1\}}$ be the per-regime local PDAGs for each $S \in \Sadm$ and regime $c \in \{0, 1\}$, and let $H_\mathrm{obs}$ be the global observational PDAG (generated by a standard constraint- or score-based algorithm applied to the observational regime). Define the per-regime background knowledge,
\[
\mathcal{K}_\mathrm{per} := \DirE(\redGtest(G)), \qquad
H_\mathrm{per} := \MPDAG(H_\mathrm{obs} \ \text{subject to} \ \mathcal{K}_\mathrm{per})
\]
where $\DirE(\cdot)$ denotes the operator for collecting all directed edges in a graph.

Let $\mathrm{ContR}(\cdot)$ be the operator for collecting directed edges output by contrastive orientation rules such as SSI/CVT/DPT. Define the contrastive-rule background knowledge,
\[
\mathcal{K}_\mathrm{rules}
:=
\bigcup_{S\in \Sadm}
\mathrm{ContR}\!\left(E_S^{(0)}, E_S^{(1)}, \{\Inv(v \mid Z)\}_{(v,Z)\in \Tadm}\right),
\]
and the corresponding aggregation,
\[
H_\mathrm{ctr}:= \MPDAG(H_\mathrm{obs}\ \text{subject to}\ \mathcal{K}_\mathrm{per} \cup \mathcal{K}_\mathrm{rules}).
\]

\begin{definition}[Non-contrastive, subset-respecting aggregators]\label{def:non-ctr-agg}
Mapping $\mathcal{A}: \{ E_S^{(c)} \}_{S \in \Sadm, c \in \{0, 1\}} \rightarrow \PDAG(V)$ is non-contrastive and subset-respecting if:
\begin{enumerate}
    \item every directed edge in $\mathcal{A}$ follows from per-regime local PDAG information and Meek closure,
    \[
    \DirE(\mathcal{A}) \subseteq \DirE(H_\mathrm{per}).
    \]
    \item $\mathcal{A}$ is invariant to swapping regime labels $(0 \leftrightarrow 1)$.
\end{enumerate}
\end{definition}

\subsection{Monotone enrichment, soundness, and separation (population)}
\begin{definition}[Contrastively compelled edges]\label{def:R}
Let $R := \DirE(\Gtest) \setminus \DirE(\redGtest)$ denote the set of directed edges compelled in $\Gtest$ but not in $\redGtest$.
\end{definition}

\begin{lemma}[Impossibility for non-contrastive subset aggregation]\label{lem:noncontrastive-impossible}
Fix a two-regime SCM class $\mathcal{M}$ satisfying Assumptions~\ref{ass:finite-scope} 
and \ref{ass:change-faithfulness}. Assume moreover that the per-regime aggregation is closure-sound with respect to the reduced target, i.e.
\[
\DirE(H_\mathrm{per}) \subseteq \DirE(\redGtest).
\]
If $R \neq \varnothing$, then for any $e \in R$, any 
non-contrastive, subset-respecting aggregator $\mathcal{A}$ that is sound for this 
model class must leave $e$ unoriented.
\end{lemma}

\begin{proof}[Proof]
Fix $e = (u \to v) \in R$. By Definition~\ref{def:R}, $e \in \DirE(\Gtest)$ 
and $e \notin \DirE(\redGtest)$.

Let $\mathcal{A}$ be any non-contrastive, subset-respecting aggregator. By 
Definition~\ref{def:non-ctr-agg},
\[
\DirE(\mathcal{A}) \subseteq \DirE(H_{\mathrm{per}}).
\]
By the assumed per-regime closure-soundness,
\[
\DirE(H_{\mathrm{per}}) \subseteq \DirE(\redGtest).
\]
Hence $e \notin \DirE(H_{\mathrm{per}})$, and therefore $e \notin \DirE(\mathcal{A})$.
Thus any sound non-contrastive, subset-respecting aggregator $\mathcal{A}$ must leave $e$ unoriented.
\end{proof}

\begin{corollary}[Reduced identifiability ceiling for non-contrastive subset aggregators]\label{cor:red-ident}
It follows from Lemma~\ref{lem:noncontrastive-impossible} that any sound non-contrastive, subset-respecting aggregator $\mathcal{A}$ cannot be uniformly guaranteed to orient edges in
\[
R=\DirE(\Gtest)\setminus \DirE(\redGtest).
\]
Equivalently, the maximal universally sound directed-edge target available from per-regime subset information alone is bounded by $\redGtest$.
\end{corollary}

\begin{theorem}[Contrastive aggregation guarantees and $\Gtest$-soundness]\label{thm:ctr-dom-soundness}
Working at the per-regime PDAG level with access to $\{E_S^{(c)}\}_{S \in \Sadm, c \in \{ 0, 1 \}}$ and $\{ \mathrm{Inv}(v \mid Z) \}_{(v, Z) \in \Tadm}$ for a testing family $(\Sadm, \Tadm)$, under Assumptions~\ref{ass:change-faithfulness} and ~\ref{ass:closure-sound}, and assuming
\[
\DirE(H_\mathrm{per}) \subseteq \DirE(\redGtest),
\]
the following hold:
\begin{enumerate}
    \item \textbf{Monotone enrichment:} $\DirE(H_\mathrm{obs}) \subseteq \DirE(H_\mathrm{per}) \subseteq \DirE(H_\mathrm{ctr})$.
    \item $\Gtest$-\textbf{soundness:} $\DirE(H_\mathrm{ctr}) \subseteq \DirE(\Gtest)$.
    \item \textbf{Separation:} If $R \neq \varnothing$, then for any \emph{sound} non-contrastive, subset-respecting aggregator $\mathcal{A}$,
    \[
    \DirE(\mathcal{A}) \subsetneq \DirE(\Gtest)
    \qquad\text{and}\qquad
    \DirE(H_\mathrm{per}) \subsetneq \DirE(\Gtest).
    \]
\end{enumerate}
\end{theorem}

\begin{proof}[Proof]
(1) By definition, $H_\mathrm{obs}$ is the global MPDAG obtained without adding any directed background knowledge beyond Meek closure. Incorporating per-regime knowledge $\mathcal{K}_\mathrm{per}$ yields
\[
H_{\mathrm{per}} = \MPDAG(H_{\mathrm{obs}}\ \text{subject to}\ \mathcal{K}_{\mathrm{per}}).
\]
Similarly, incorporating the contrastive-rule directions gives
\[
H_\mathrm{ctr} = \MPDAG(H_{\mathrm{obs}}\ \text{subject to}\ \mathcal{K}_{\mathrm{per}} \cup \mathcal{K}_{\mathrm{rules}}).
\]
By monotonicity of the $\MPDAG(\cdot)$ operator under consistent background knowledge, since
\[
\varnothing \subseteq \mathcal{K}_\mathrm{per} \subseteq \mathcal{K}_\mathrm{per}\cup\mathcal{K}_\mathrm{rules},
\]
it follows that
\[
\DirE(H_\mathrm{obs}) \subseteq \DirE(H_\mathrm{per}) \subseteq \DirE(H_\mathrm{ctr}).
\]

(2) Since $\mathcal{K}_{\mathrm{per}} := \DirE(\redGtest(G))$, and by Definitions~\ref{def:red-res-psi-equivalence} 
and~\ref{def:red-G-test}, $\DirE(\redGtest(G)) \subseteq \DirE(\Gtest)$, our original equivalence class $[G]_{\text{res}}$ is a refinement of the reduced equivalence class $[G]_{\Sadm}^{\text{res}}$. Hence, it follows that $\mathcal{K}_{\mathrm{per}} \subseteq \DirE(\Gtest)$. 
By Propositions~\ref{prop:ssi-psi}, \ref{prop:cvt-psi}, and~\ref{prop:dpt-psi}, 
\[
\mathcal{K}_{\mathrm{rules}} \subseteq \DirE(\Gtest).
\]
Hence
\[
\mathcal{K}_{\mathrm{per}} \cup \mathcal{K}_{\mathrm{rules}} \subseteq \DirE(\Gtest),
\]
and by Assumption~\ref{ass:closure-sound}, Meek closure does not orient any edges contradicting $\Gtest$, thus
\[
\DirE(H_{\mathrm{ctr}}) \subseteq \DirE(\Gtest).
\]

(3) If $R \neq \varnothing$, pick any $e \in R$. By Definition~\ref{def:R}, $e \in \DirE(\Gtest)$ and $e \notin \DirE(\redGtest)$. By Lemma~\ref{lem:noncontrastive-impossible}, for any sound non-contrastive, subset-respecting aggregator $\mathcal{A}$, we have $e \notin \DirE(\mathcal{A})$. Since $e \in \DirE(\Gtest)$, this implies
\[
\DirE(\mathcal{A}) \subsetneq \DirE(\Gtest).
\]
Moreover, by the assumed per-regime closure-soundness
\[
\DirE(H_\mathrm{per}) \subseteq \DirE(\redGtest),
\]
and since $e \notin \DirE(\redGtest)$, it follows that $e \notin \DirE(H_\mathrm{per})$. Together with $e \in \DirE(\Gtest)$, this yields
\[
\DirE(H_\mathrm{per}) \subsetneq \DirE(\Gtest).
\]
\end{proof}

\section{Asymptotic recovery under subset sampling}
In practice, we sample subsets $S_1,\dots,S_T$ from a (possibly biased/adaptive) sampler $Q$ over $\Sadm$.
For an edge $e$ (or a direction $u\to v$) define the \emph{witness family} $\Wit_e\subseteq\Sadm$ as the set of
admissible subsets on which $e$ is recovered at the population level either by a local PDAG orientation or by an application of the orientation rules. Let $\pi_e:=Q(\Wit_e)$.

\begin{lemma}[Coverage tail for i.i.d.\ subset sampling]\label{lem:coverage-tail}
Assume the sampled subsets $S_1,\dots,S_T$ are i.i.d.\ from $Q$. Fix $e$ with witness mass $\pi_e>0$ and let
$L_e:=\sum_{t=1}^T \ind\{S_t\in \Wit_e\}$. Then for any integer $r\in\{1,\dots,T\}$,
\[
\Prb(L_e<r)\ =\ \sum_{\ell=0}^{r-1} \binom{T}{\ell}\pi_e^\ell(1-\pi_e)^{T-\ell}
\ \le\ \exp\!\Big(-T\,\Dkl(r/T\ \|\ \pi_e)\Big),
\]
where $\Dkl(a\|b):=a\log\frac{a}{b}+(1-a)\log\frac{1-a}{1-b}$ is the binary KL divergence.
\end{lemma}

\begin{proof}
Let
$
H_t =
\begin{cases}
1 & \text{if } S^{(t)} \in S_e^{\dagger} \\
0 & \text{otherwise }
\end{cases}
$
and $\pi_e = \mathcal Q(S_e^{\dagger})$, then $H_t$ can be represented as a Bernoulli trial with success corresponding to a good hit for $e$, $H_t \sim \text{Bernoulli}(\pi_e)$. 
\\
\\
Since $S^{(t)}$ are i.i.d. samples from $\mathcal Q$, Bernoulli trials $H_t$ (where $t \in \{1,...,T\}$) are independent.
\\
\\
Let $L_e:=\sum_{t=1}^T \mathbf{1}\{S^{(t)}\in\mathcal S_v^\dagger\} = \sum_{t=1}^T H_t$. Then, $L_e$ is the sum of T independent Bernoulli trials and therefore, $L_e \sim \text{Binomial}(T, \pi_e)$.
\\
\\
We can calculate $\Pr(L_e < r_e)$ for any integer $r_e \in \{1,...,T\}$ as a sum of probabilities for all $L_e = l$ s.t. $0 \leq l < r_e$: $\Pr(L_e<r_e)\ =\ \sum_{\ell=0}^{r_e-1}\binom{T}{\ell}\pi_e^\ell(1-\pi_e)^{T-\ell}$. \\
\\
Since we want to find the upper bound for the CDF $\Pr(L_e<r_e)$, we can rewrite it as $\Pr(tL_e > tr_e)$, where $t<0$ and use Chernoff bound to upper bound the tail of the distribution.
\\
\\
First, we can write out the Markov inequality:
\[
\Pr(e^{tL_e} \geq e^{tr_e}) \leq \frac{\mathbb{E}[e^{tL_e}]}{e^{tr_e}}\text{, where } t<0
\]
\[
\mathbb{E}[e^{tL_e}] = (1-\pi_e + \pi_e e^t)^T
\]
\[
\Pr(L_e < r_e) \leq \frac{(1-\pi_v + \pi_e e^t)^T}{e^{tr_e}} = \exp(T \log(1-\pi_e + \pi_ee^t)-tr_e)
\]
\\
Let $f(t) = T \log(1-\pi_e+\pi_ee_t)-tr_e)$
\\
\\
To minimize the right hand side of the inequality:
\[
\frac{d}{dt}f(t) = \frac{T\pi_ee^t}{1-\pi_e+\pi_e e^t} - r_v = 0
\]
\[
\frac{\pi_ee^t}{1-\pi_e+\pi_ee^t} = \frac{r_e}{T}
\]
\[
e^t = \frac{\frac{r_e}{T}(1-\pi_e)}{\pi_e(1-\frac{r_e}{T})}
\]
\[
t = log(\frac{\frac{r_e}{T}(1-\pi_e)}{\pi_e(1-\frac{r_e}{T})})
\]
\\
\\
Plug these back into $f(t)$:\\
\begin{align*}
    f(t)
    &= T\log(\frac{1-\pi_e}{1-\frac{r_e}{T}}) - r_e \log\frac{\frac{r_e}{T}(1-\pi_e)}{\pi_e(1-\frac{r_e}{T})} \\
    & = T\log(\frac{1-\pi_e}{1-\frac{r_e}{T}}) - r_e\log\frac{\frac{r_e}{T}}{\pi_e} - r_e\log\frac{1-\pi_e}{1-\frac{r_e}{T}} \\
    & = (T-r_e)\log\frac{1-\pi_e}{1-\frac{r_e}{T}} - r_e(\log\frac{\frac{r_e}{T}}{\pi_e})\\
    & = -T\left((1-\frac{r_e}{T})\log\frac{1-\frac{r_e}{T}}{1-\pi_e}+\frac{r_e}{T}\log\frac{\frac{r_e}{T}}{\pi_e}\right) \\
    & = -TD(\frac{r_e}{T} || \pi_e)
\end{align*}
where $D(a||b) = a\log\frac{a}{b} + (1-a)\log\frac{1-a}{1-b}$ is the binary KL.
\\
\\
As a result, 
\[
\Pr(L_e<r_e)\ =\ \sum_{\ell=0}^{r_e-1}\binom{T}{\ell}\pi_e^\ell(1-\pi_e)^{T-\ell} \leq \text{exp}(-TD(\frac{r_e}{T} || \pi_e))
\]
\end{proof}

\begin{assumption}[Witness coverage for $\Gtest$]\label{ass:witness-coverage-psi}
For every compelled directed edge $u\to v$ in $\Gtest$, there exists a nonempty witness family $\Witu\subseteq\Sadm$
such that on each $S\in\Witu$ the direction $u\to v$ is recovered at the population level by either a local PDAG
orientation or one of the orientation rules, and the sampler assigns positive mass $\pi_{u\to v}:=Q(\Witu)>0$.
\end{assumption}

\begin{theorem}[Consistency for the test-induced (restricted $\Psi$) essential graph]\label{thm:main-psi}
Let $\widehat G$ be the model output after $T$ sampled subsets:
(i) estimate $\widehat E_S^{(c)}$ on each sampled $S$ and each regime $c$,
(ii) estimate $\hInv(v\mid Z)$ for tested $(v,Z)$ pairs on sampled witnesses,
(iii) apply orientation rules using these estimates, and
(iv) aggregate all inferred constraints with closure $\Cl$ to output $\widehat G$.
Under Assumptions~\ref{ass:finite-scope}--\ref{ass:closure-sound} and witness coverage
(Assumption~\ref{ass:witness-coverage-psi}), as $n,T\to\infty$,
\[
\Prb\!\left(\widehat G = \Gtest\right)\ \to\ 1.
\]
\end{theorem}

\begin{proof} Let $\mathcal E_n$ be the event that all invoked local PDAGs and invariance tests used by the model are correct over the finite admissible families:

\[
\mathcal E_n := 
\Big\{\forall S\in\Sadm,\forall c\in\{0,1\}: \widehat E_S^{(c)} = E_S^{(c)}\Big\}
\cap
\Big\{\forall (v,Z)\in\Tadm: \hInv(v\mid Z)=\Inv(v\mid Z)\Big\}
\]
\[
\Prb\big(\neg \mathcal E_n) 
= 
\sum_{S, c}\Prb(\widehat E_S^{(c)} \neq E_S^{(c)}) 
+ 
\sum_{(v, Z)} \Prb(\hInv(v\mid Z) \neq \Inv(v\mid Z))
\]
\\
By Assumptions \ref{ass:local-pdag-cons} and \ref{ass:inv-cons}, each of these terms $\to 0$ as $n \to \infty$. By Assumption~\ref{ass:finite-scope}, both $\Sadm$ and $\Tadm$ are finite. As a result, $\Prb(\neg\mathcal E_n) \to 0$ as $n \to \infty$.\\
\[
\Prb(\mathcal E_n) = 1 - \Prb(\neg \mathcal E_n)
\]
Since $\Prb(\neg \mathcal E_n) \to 0$ as $n \to \infty$, $\Prb( \mathcal E_n) \to 1$ as $n \to \infty$.

Given that $\mathcal E_n$ holds, by Definition ~\ref{def:res-psi-equivalence}, $\Psi^{\mathrm{res}}_{\Sadm,\Tadm} (\widehat{G}) = \Psi^{\mathrm{res}}_{\Sadm,\Tadm} (G)$, for some ground truth graph $G$. Then, $\widehat{G} \in [G]_{\mathrm{res}}$. By $\Gtest$-soundness in Theorem ~\ref{thm:ctr-dom-soundness}, on the event $\mathcal E_n$, $\mathrm{DirE}(\widehat{G}) \subseteq \mathrm{DirE}(\Gtest)$. Since, $\Prb(\mathcal E_n) \to 1$ as as $n \to \infty$,
\[
\Prb(\mathrm{DirE}(\widehat{G}) \subseteq \mathrm{DirE}(\Gtest)) \to 1
\]
as $n \to \infty$.\\
\\
For any compelled edge $e$ ($u \to v$) in $\Gtest$, by Assumption~\ref{ass:witness-coverage-psi}, there exists a nonempty witness family $\Wit_e\subseteq\Sadm$ with $\pi_{e}=Q(\Wit_e)>0$. By Lemma~\ref{lem:coverage-tail}, the probability of the samples not hitting a witness can be expressed as 
\\
\[
\Prb(L_e = 0) = (1-\pi_e)^T \to 0 \text{ as } T \to \infty
\]
since $\pi_e > 0$. Then the probability of the sampler hitting a witness is $\Prb(L_e \geq 1) = 1 - \Prb(L_e = 0) \to 1 \text{ as } T \to \infty$. \\
\\
On $\mathcal{E}_n \cap \{L_e \geq 1\}$, there exists some t such that $S_t \in \Wit_e$. By the definition of witness family $\Wit_e$, the edge $e$ is recovered on $S_t$ and inserted into the aggregation at the population level. By Assumption ~\ref{ass:closure-sound}, closure does not introduce orientations that contradict $\Gtest$. Then, $e \in \DirE(\widehat G)$ on $\mathcal{E}_n \cap \{L_e \geq 1\}$. Therefore,
\[
\mathcal{E}_n \cap \{L_e \geq 1\} \subseteq \{e \in \DirE(\widehat G) \}
\]
\[
e \notin \DirE(\widehat{G}) \subseteq \neg \mathcal{E}_n \cup L_e=0
\]
Define the event $\mathcal A_e$ such that
\[ 
\mathcal A_e := \{e \notin \DirE(\widehat{G}) \}
\]
Apply union bound
\[
\Prb \Big(\bigcup_{e \in \DirE(\Gtest)} \mathcal{A}_e \Big) 
\leq 
\sum_{e \in \DirE(\Gtest)} \Prb(\mathcal{A}_e) 
\leq 
\sum_{e \in \DirE(\Gtest)} \Prb(\neg \mathcal E_n) + \sum_{e \in \DirE(\Gtest)} \Prb(L_e = 0) 
\]

As $n, T \to \infty \text{, } \Prb(\neg \mathcal E_n) \to 0 \text{ and } \Prb(L_e = 0) \to 0$ and therefore, the right hand side of the inequality $\to 0$. Consequently, $\Prb(\forall e \in \DirE(\Gtest)\ : e \in \DirE(\widehat G)) \to 1$. Therefore
\[
\Prb (\DirE(\Gtest) \subseteq \DirE(\widehat{G})) \to 1
\]

As a result, $\Prb(\widehat G = \Gtest) \to 1$ as $n, T \to \infty$.
\end{proof}

\section{Model Architecture} \label{appendix::arch}
We highlight the complete architecture details for SCONE in the below sections. 

\subsection{Constructing Marginal Causal Structures}
Let $V = \{ 1, ..., N \}$ be the nodes in the DAG $G = (V, E)$. We form $T$ admissible subsets $S_1, ..., S_T \subseteq V$ and fix a global catalog of valid candidate undirected edges $e_k = \{ i, j\}$ for $k = 1,..., K$ unique unordered node pairs making up each edge. We consider two regimes $c \in \{ 0, 1\}^N$.

\subsubsection{Sampling admissible subsets of nodes}\label{appendix:arch-sampling}
We first construct selection scores that prioritize soft intervention targets using cross-regime contrast while ensuring graph coverage. We then propose a sampling approach that greedily optimizes selection scores to sample admissible subsets of nodes.

\paragraph{Selection scores} Let $c \in \{ 0, 1\}^N$ be the environment labels. We define node-wise environment sensitivity scores $s_i \in [0, 1]^N$, which combines standardized mean shift and dependence on the environment:
\begin{equation}\label{eq:arch-sensitivity-scores}
s_i = \mathrm{rescale} \left ( w_1 \frac{|\hat{\mu}_i^{(0)} - \hat{\mu}_i^{(1)}|}{\sqrt{\frac{1}{2}\left (\hat{\sigma}_i^{(0)2} - \hat{\sigma}_i^{(1)2} \right )}} + w_2 \mathrm{MI}(X_i;c) \right ), w_1, w_2 \geq 0,
\end{equation}
where $\hat{\mu}_i^{(c)}$ and $\hat{\sigma}_i^{(c)}$ are the sample mean and standard deviation of $X_i$ in environment $c \in \{ 0, 1\}$, $\mathrm{MI}$ is the mutual information between $X_i$ and $c$. We also define pairwise contrast scores $\Gamma \in [0, 1]^{N \times N}$ as:
\begin{equation}
\Gamma_{i, j} := \mathrm{rescale}\left (| \mathrm{corr}^{(0)}(X_i, X_j) - \mathrm{corr}^{(1)}(X_i, X_j)| \right).
\end{equation}
As nodes are sampled, we update the selection scores to downweight reused pairs. Let $n_{i,j}^t$ be the number of times the pair $(i, j)$ has appeared in $S_1,...,S_{t - 1}$ and let $n_i^t$ be the number of times node $i$ has appeared. We downweight selection scores $\alpha_{i,j}^t$ and $\Gamma^t_{i,j}$ for pair $(i,j)$ for subset $t$ by visit frequency via:
\[
\alpha_{i, j}^t = \frac{\alpha_{i,j}}{q^{n_{i, j}^t}}, q>1; \qquad \Gamma_{i,j}^t = \frac{\Gamma_{i,j}}{r^{n_{i,j}^t}}, r>1,
\]
where $\alpha_{i,j}$ are global statistics such as inverse covariance or precision.

\paragraph{Greedy selection} We greedily sample nodes to form subsets $S_t$ based on selection scores $\alpha^t$ and $\Gamma^t$. Fix hyperparameters $\rho \in [0, 1]$ $\lambda \in [0, 1]$, and $\beta_{\mathrm{cov}} \geq 0$. We initialize $S_t$ with a seed by drawing a single node with probability:
\begin{equation}
S_t \leftarrow \{ i: i \sim \mathrm{Categorical}(\tilde{\alpha}_1^t, ..., \tilde{\alpha}_N^t)\}, \qquad \tilde{\alpha}_i^t := \left [ (1 - \rho) \sum_{j \neq i} \alpha_{i,j}^t + \rho s_i \right ] \cdot \frac{1}{(1 + n_{i,j}^t)^{\beta_\mathrm{cov}}}.
\end{equation}
While $|S_t| < k$, we add one node at a time with probability:
\begin{equation}
S_t \leftarrow S_t \cup \left \{ j : j \sim \mathrm{Categorical}(\tilde{\alpha}_{1, S_t}^t, ..., \tilde{\alpha}_{N, S_t}^t) \right \},
\end{equation}
where 
\begin{equation}
\tilde \alpha^{\,t}_{j,S_t}
\ :=\
\lambda \sum_{i\in S_t}\alpha^{\,t}_{i,j}
\ +\
(1-\lambda)\left( s_j\ +\ \frac{1}{|S_t|}\sum_{i\in S_t}\Gamma^{\,t}_{i,j}\right)
\quad \text{if } j\notin S_t,\qquad
\tilde \alpha^{\,t}_{j,S_t}\ :=\ 0 \text{ otherwise,}
\end{equation}
followed by a coverage bonus $\tilde{\alpha}_{j,S_t}^t \leftarrow \tilde{\alpha}_{j,S_t}^t / (1 + n_j^t)^{\beta_\mathrm{cov}}$. After forming $S_t$ we update the visit counts $n_i^{t + 1} \leftarrow n_i^t + \mathbf{1}\{ i \in S_t\}$ and pair counts $n_{i,j}^t \leftarrow n_{i,j}^t + \mathbf{1}\{ i,j \in S_t\}$, then refresh $\alpha^{t + 1}$ and $\Gamma^{t + 1}$ via the formulas above. 

\subsubsection{Classical causal discovery ensemble}\label{appendix:arch-ensemble}
Given node observations $X \in \mathbb{R}^{N \times S}$ for $N$ nodes and $S$ samples, the causal discovery module produces a distribution over edge classes $\{ i \to j, j \to i, i - j, \text{no-edge}\}$ for each node pair $(i,j)$. 

\paragraph{Bootstrap replicates and vote decoding} For each of the $R_\mathrm{boot}$ replicates, we draw $X^{(r)} \in \mathbb{R}^{N \times M}$ by resampling $M$ observations with replacement. Each replicate is passed to each base learner $l \in \{ 1, ..., L\}$ to produce a score matrix $\mathbf{A}^{(l, r)} \in [0, 1]^{N \times N}$, where $\mathbf{A}^{(l, r)}[i, j]$ estimates $P(i \to j)$. Let the hyperparameters $\varepsilon$ be the tie-breaking margin and $\tau_l$ a per-learner confidence threshold. Each score matrix is then decoded into a one-hot vote $v_{i,j}^{(l, r)} \in \{ 0, 1 \}^4$ as:
\[
v_{i,j}^{(l, r)} := \begin{cases}
i \to j & \text{if} \ \mathbf{A}^{(l, r)}[i, j] - \mathbf{A}^{(l, r)}[j, i] > \varepsilon \ \text{and}\ \max(\mathbf{A}^{(l, r)}[i, j], \mathbf{A}^{(l, r)}[j, i]) \geq \tau_l, \\
j \to i & \text{if}\ \mathbf{A}^{(l,r)}[j, i] - \mathbf{A}^{(l, r)}[i, j] > \varepsilon \ \text{and}\ \max(\mathbf{A}^{(l, r)}[i, j], \mathbf{A}^{(l, r)}[j, i] \geq \tau_l,\\
i - j & \text{if}\ \max(\mathbf{A}^{(l, r)}[i, j], \mathbf{A}^{(l, r)}[j, i] \geq \tau_l\ \text{and neither direction dominates},\\
\text{no-edge} & \text{if}\ \max(\mathbf{A}^{(l, r)}[i, j], \mathbf{A}^{(l, r)}[j, i] < \tau_l
\end{cases}
\]
\paragraph{Majority voting}
Weighted vote counts $\mathbf{W}[i,j] \in \mathbb{R}^4$ are accumulated across all learners and replicates:
\[
\mathbf{W}[i,j] = \sum_{l = 1}^L w_l \sum_{r = 1}^{R_\mathrm{boot}} v_{i,j}^{(l, r)}
\]
The final edge class is determined by:
\begin{equation}
e_{i,j}:= \begin{cases}
\argmax_{(i,j)}(\mathbf{W}[i, j]) & \text{if the argmax is unique},\\
\text{no-edge} & \text{otherwise} \ 
\end{cases}
\end{equation}
and is outputted as a one-hot encoding for the edge class, $\hat{e}_{i,j} = \mathbf{1} \{e_{i,j}\} \in \{ 0, 1 \}^4$.
\paragraph{PolyBIC for pairwise orientation} We primarily use PolyBIC as our base causal discovery learner. For each pair PolyBIC fits polynomial ridge regressions of degree $D$ in both directions and selects the causal direction as the one yielding the lower Bayesian Information Criterion (BIC) over the residuals. The orientation score is:
\begin{equation}
\mathbf{A}^{(l, r)}[i, j] = \sigma \left ( \frac{\mathrm{BIC}(X_i \mid X_j) - \mathrm{BIC}(X_j \mid X_i)}{\mathbb{E}_b \left [|\mathrm{BIC}^b(X_i \mid X_j) - \mathrm{BIC}^b(X_j \mid X_i) |\right]} \right )
\end{equation}
where $\sigma$ is the sigmoid function and $b$ is batch. A higher orientation score $\mathbf{A}^{(l, r)}[i, j]$ favors the orientation $i \to j$.

\subsubsection{Edge Tokenization}
\label{appendix:arch-tokenization}

\paragraph{Candidate Edge Construction}
For each graph, the ensemble CD module runs independently on both environments for each sampled subset $S_t$ of size $k$, producing per-environment local CPDAGs whose directed and undirected edges are lifted to global node indices and deduplicated. Additionally, the top-$K_{\text{env}}$ node pairs by absolute correlation difference $|\text{corr}_1(i,j) - \text{corr}_0(i,j)|$ are included as environment-contrast candidates. The union of CD edges and environment-contrast edges forms the candidate set $K_{\text{cand}}$ of $K$ edges, optionally capped at a maximum size.

\paragraph{Endpoint Encoding}
Each candidate edge $e = \{i,j\}$ is represented per environment $c \in \{0,1\}$ by a raw feature vector $x_{e,c} \in \mathbb{R}^{6 + 2N + F}$ formed by concatenating three components. First, a 6-dimensional endpoint encoding $\mathbf{o}_{e,c} \in \{0,1\}^6$ encodes the CPDAG edge type at each endpoint as a one-hot over three states $\{\varnothing, \bullet, \triangleright\}$, representing no edge, tail, and arrowhead respectively:
\begin{equation}
\mathbf{o}_{e,c} = [\mathbf{1}_{\varnothing}(i),\ \mathbf{1}_{\bullet}(i),\ \mathbf{1}_{\triangleright}(i),\ \mathbf{1}_{\varnothing}(j),\ \mathbf{1}_{\bullet}(j),\ \mathbf{1}_{\triangleright}(j)]
\end{equation}
Second, one-hot node identifiers $\mathbf{1}_N(i), \mathbf{1}_N(j) \in \{0,1\}^N$ provide global node identity. Third, per-environment pairwise statistics $\phi^{(c)}_{ij} \in \mathbb{R}^F$ with $F=4$ contain:
\begin{equation}
\phi^{(c)}_{ij} = \left[\hat{\rho}^{(c)}_{ij},\ -\hat{\Omega}^{(c)}_{ij},\ \tanh\!\left(\frac{\hat{\sigma}^{(c)}_{ij\to i}}{\hat{\sigma}^2_i}\right),\ \tanh\!\left(\frac{\hat{\sigma}^{(c)}_{ij\to j}}{\hat{\sigma}^2_j}\right)\right]
\end{equation}
where $\hat{\rho}^{(c)}_{ij}$ is the sample correlation, $\hat{\Omega}^{(c)}_{ij}$ is the normalized partial correlation from the precision matrix, and the last two entries are tanh-scaled regression coefficients in each direction.

\paragraph{Marginal Edge Embeddings}
The raw feature vectors are projected to dimension $d$ via a two-layer MLP and augmented with sinusoidal positional embeddings along both the subset and edge axes:
\begin{equation}
H_E[b, t, e, c, :] = \mathrm{MLP}(x_{e,c}) + \mathrm{PE}_T(t) + \mathrm{PE}_K(k)
\end{equation}
where $\mathrm{PE}_T$ and $\mathrm{PE}_K$ are sinusoidal embeddings masked to valid subset and edge positions respectively, yielding $H_E \in \mathbb{R}^{B \times T \times K \times 2 \times d}$.

\paragraph{Global Stream Construction}
In parallel, the global stream $H_\rho \in \mathbb{R}^{B \times N \times N \times d}$ is constructed from the symmetrized precision matrix of the environment-averaged node features:
\begin{equation}
\bar{X} = \frac{1}{2}(X^{(0)} + X^{(1)}), \qquad \hat{\Omega} = \frac{1}{2}\left(\Sigma^{-1}(\bar{X}) + \Sigma^{-1}(\bar{X})^\top\right)
\end{equation}
Each entry is projected to dimension $d$ and summed with learned pairwise node positional embeddings:
\begin{equation}
H_\rho[b, i, j, :] = W_\rho\,\hat{\Omega}[b, i, j] + \mathbf{p}_i + \mathbf{p}_j
\end{equation}
where $W_\rho \in \mathbb{R}^{1 \times d}$ is a learned linear projection and $\mathbf{p}_i, \mathbf{p}_j \in \mathbb{R}^d$ are learned node positional embeddings. The diagonal of $\hat{\Omega}$ is zeroed before projection.

\paragraph{Invariant/Contrast Reparametrization}
Before the first axial block, marginal edge embeddings $H_E \in \mathbb{R}^{B \times T \times K \times 2 \times d}$ are reparametrized into invariant and contrast channels:
\begin{equation}
z_{\text{avg}} = \frac{1}{2}(H_E[\cdot,\cdot,\cdot,1,:] + H_E[\cdot,\cdot,\cdot,0,:]), \qquad z_{\Delta} = H_E[\cdot,\cdot,\cdot,1,:] - H_E[\cdot,\cdot,\cdot,0,:]
\end{equation}
replacing the original environment channels as the two operative channels $[z_{\text{avg}}, z_{\Delta}]$ to form $\widetilde{H}_E$. This bijective reparametrisation disentangles regime-invariant structure from regime-specific variation before attention is applied.

\subsection{Bias heads for contrastive orientation} \label{appendix:bias-heads}
Each contrastive orientation rule produces signed biases over three orientation classes $\{ i \to j, j \to i, i - j\}$ for each candidate edge. To measure how much a node's mechanism shifts across regimes, we define the invariance score:
\begin{equation}
\gamma_v = \mathrm{Norm} \left ( \sum_{e \ni v} ||\Delta_e|| \right ) \in [0, 1], g_v:= 1 - \gamma_v
\end{equation}
where $\Delta_{i,j} = \widetilde{H}_E[t, \{i, j\}, 1, :] - \widetilde{H}_E[t, \{ i, j\}, 0, :]$ is the regime contrast for edge $e = \{ i, j \}$ and $g_v$ denotes its complement. Nodes with high $g_v$ are more likely regime-invariant (non-targets) and thus provide stronger orientation signal. 

\subsubsection{Single-Sided Invariance (SSI)}\label{appendix:arch-ssi}
For each edge $\{i,j\}$, parent contexts $\mathbf{p}_i, \mathbf{p}_j \in \mathbb{R}^d$ are formed by mean-pooling embeddings of incident candidate edges appearing in at least a fraction $\rho$ of subsets:
\[
\mathbf{p}_v := \mathrm{Pool}\{ \widetilde{H}_E[t, \{ v, u\}, c, :]: u \in \mathrm{Pa}(v)\}
\]
A shared predictor $\psi$ maps each endpoint embedding and its parent context to a sufficient-statistic vector:
\begin{align}
\phi^{(c)}_j &= \psi\!\Big(\widetilde{H}_E[t,\{i,j\},c,:]\ ||\ \mathbf{p}_j\ ||\ \mathbf{m}_j\Big),\\
\phi^{(c)}_i &= \psi\!\Big(\widetilde{H}_E[t,\{j,i\},c,:]\ ||\ \mathbf{p}_i\ ||\ \mathbf{m}_i\Big),
\end{align}
and $\delta_v = \|\phi^{(1)}_v - \phi^{(0)}_v\|$ measures the cross-environment shift. A bias toward $i\!\to\!j$ is applied when $j$ shifts more than $i$:
\begin{equation}
b^{\text{SSI}}_{t,\{i,j\}}(i\!\to\!j) \mathrel{+}= \kappa \cdot \mathrm{clip}(\delta_j - \delta_i,\ -\tau,\ \tau)^+, \quad
b^{\text{SSI}}_{t,\{i,j\}}(j\!\to\!i) \mathrel{-}= \kappa \cdot \mathrm{clip}(\delta_j - \delta_i,\ -\tau,\ \tau)^+
\end{equation}
gated by $\gamma_j \geq \gamma_i + \eta$ and $\gamma_i \leq \gamma_{\max}$, ensuring the bias fires only when $j$ is the more perturbed endpoint.

\subsubsection{Contrastive V-structure (CVT)}\label{appendix:arch-cvt}
For each unshielded triple $(i,j,k)$ in $S_t$ (i.e., $\{i,j\},\{k,j\}\in K_\mathrm{cand}$ and $\{i,k\}\notin K_\mathrm{cand}$), a collider score is computed from the environment-1 embeddings and cross-environment contrasts of both incident edges:
\begin{equation}
s^{(1)}_{\mathrm{cvt}}(i\!\to\!j\!\leftarrow\!k)
= \sigma\!\left(u^\top\Big[\widetilde{H}_E[t,\{i,j\},1,:]\ \|\ \widetilde{H}_E[t,\{k,j\},1,:]\ \|\ \Delta_{i,j}\ \|\ \Delta_{k,j}\Big]\right)
\end{equation}
where $u$ is learned and $\Delta_{v,u} = \widetilde{H}_E[t,\{v,u \},1,:] - \widetilde{H}_E[t,\{v,u \},0,:]$. Scores below a confidence threshold $\theta_{\mathrm{cvt}}$ are zeroed. The bias is applied to condition-0 edge embeddings only when $\gamma_j \leq \theta_\gamma$, i.e.\ when the center node $j$ is sufficiently regime-invariant to be a plausible collider:
\begin{align}
b^{\text{CVT}}_{t,\{i,j\},0}(i\!\to\!j) &\mathrel{+}= \omega\cdot s^{(1)}_{\mathrm{cvt}},
&b^{\text{CVT}}_{t,\{i,j\},0}(j\!\to\!i) &\mathrel{-}= \omega\cdot s^{(1)}_{\mathrm{cvt}},\\
b^{\text{CVT}}_{t,\{k,j\},0}(k\!\to\!j) &\mathrel{+}= \omega\cdot s^{(1)}_{\mathrm{cvt}},
&b^{\text{CVT}}_{t,\{k,j\},0}(j\!\to\!k) &\mathrel{-}= \omega\cdot s^{(1)}_{\mathrm{cvt}},
\end{align}
promoting the collider orientation at $j$ while penalizing the reverse directions on both incident edges.

\subsubsection{Contrastive Discriminating Path (DPT)}\label{appendix:arch-dpt}
For each candidate edge $\{i,j\}$, the top-$B$ paths $\pi = (e_0, \dots, e_m)$ of length $m \leq L$ through the candidate edge graph are enumerated offline via BFS. Each path is encoded independently per environment using a DeepSets aggregator $\mathrm{DeepSet} = \mathrm{DeepSet}_\rho \circ \mathrm{DeepSet}_\phi$ over concatenated edge embeddings and cross-environment contrasts:
\begin{equation}
s^{(c)}_\pi = \sigma\!\left(\mathrm{DeepSet}\!\left(\left\{\big[\widetilde{H}_E[t, e_r, c, :]\ \|\ \Delta_{e_r}\big]\right\}_{r=0}^{m}\right)\right) \in [0,1]
\end{equation}
The environment with the higher score is selected as the \emph{winning} environment $c^* = \arg\max_c\, s^{(c)}_\pi$, and the same bias is reflected to the \emph{other} environment $\bar{c} = 1 - c^*$. Paths with $s^{(c^*)}_\pi < \theta_{\mathrm{dpt}}$ are discarded. The bias for the owner edge $\{i,j\}$ is then:
\begin{align}
b^{\text{DPT}}_{t,\{i,j\},\bar{c}}(i\!\to\!j) &\mathrel{+}= \lambda\cdot s^{(c^*)}_\pi, \qquad 
b^{\text{DPT}}_{t,\{i,j\},\bar{c}}(j\!\to\!i) \mathrel{-}= \lambda\cdot s^{(c^*)}_\pi,
\end{align}
aggregated across all candidate paths for the edge.

\subsection{Axial Aggregator}
\label{appendix:arch-aggregator}

\subsubsection{Axial Block}
Each of the $L$ axial blocks processes both streams in sequence. Given local stream $H_E \in \mathbb{R}^{B \times T \times K \times C \times d}$ we first reparametrize embeddings into $\widetilde{H}_E$, axial attention is applied along the $T$ and $K$ axes by reshaping to $[T, KC, B, d]$ and applying row- and column-wise self-attention with pre-LayerNorm residual connections:
\begin{equation}
\widetilde{H}_E \leftarrow \widetilde{H}_E + \mathrm{FFN}(\mathrm{ColAttn}(\mathrm{RowAttn}(\widetilde{H}_E)))
\end{equation}

\paragraph{Invariance gating.} After axial attention, per-node invariance scores $\gamma_v \in [0,1]$ are computed from the updated local stream (see Eq.~\ref{eq:inv-score}) and scattered to edge tokens, producing a per-edge score pair $[\gamma_i, \gamma_j] \in \mathbb{R}^{B \times T \times K \times 2}$. These are passed through a learned gating network to produce per-head, per-edge gate values:
\begin{equation}
[g_{\text{cvt}}, g_{\text{ssi}}, g_{\text{dpt}}]_{t,e} = g_{\min} + (1 - g_{\min})\cdot\sigma\!\left(\mathbf{W}_{\text{gate}}\,[\gamma_i, \gamma_j]_{t,e} / \tau_{\text{gate}}\right) \in [0,1]^3
\end{equation}
where $g_{\min}$ is a minimum gate value preventing full suppression and $\tau_{\text{gate}}$ is a temperature. Each orientation head (CVT, SSI, DPT) produces a bias $b^{(\cdot)} \in \mathbb{R}^{B \times T \times K \times 2 \times |\mathcal{C}|}$, where $|\mathcal{C}|$ is the number of edge classes, which is projected back to embedding space and added to the local stream:
\begin{equation}
\widetilde{H}_E = \widetilde{H}_E + \sum_{h \in \{\text{cvt, ssi, dpt}\}} g_h \cdot W_{\text{proj}}(b^{(h)})
\end{equation}
A gate regularization term $\mathcal{L}_{\text{gate}} = \frac{1}{3}\sum_h (1 - \bar{g}_h)^2$ encourages gates to remain active throughout training, where $\bar{g}_h$ is the mean gate value across edges and subsets.

\paragraph{Edge logits and subset pooling.} Per-edge logits are computed from the updated local stream via a linear projection:
\begin{equation}
\Theta_{t,e} = W_{\text{cls}}\,\widetilde{H}_E[t, e, :, :] \in \mathbb{R}^{C \times |\mathcal{C}|}
\end{equation}
Logits are pooled across subsets $T$ using learned attention weights over valid edges:
\begin{equation}
w_{t,e} = \frac{\exp(W_{\text{pool}}\,\widetilde{H}_E[t,e,:,:])}{\sum_{t': (t',e) \text{ valid}} \exp(W_{\text{pool}}\,\widetilde{H}_E[t',e,:,:])}, \qquad \hat{\Theta}_e = \sum_t w_{t,e} \cdot \Theta_{t,e}
\end{equation}
The pooled logits $\hat{\Theta}_e \in \mathbb{R}^{C \times |\mathcal{C}|}$ are scattered into a global node-pair table $\mathcal{N} \in \mathbb{R}^{B \times N \times N \times 2 \times |\mathcal{C}|}$ by writing $\mathcal{N}[b,u,v] \leftarrow \hat{\Theta}_e$ for each candidate edge $e = \{u,v\}$.

\paragraph{Message passing.} Pooled edge messages $m_e \in \mathbb{R}^d$ (averaged over $C$) are similarly scattered into a global message table $M \in \mathbb{R}^{B \times N \times N \times d}$, with unobserved pairs defaulting to zero. The global stream then undergoes row- and column-wise axial attention over the $N \times N$ grid:
\begin{equation}
H_\rho \leftarrow H_\rho + \mathrm{FFN}(\mathrm{ColAttn}(\mathrm{RowAttn}(M)))
\end{equation}
The updated global representations are broadcast back to the marginal stream by adding the corresponding node-pair embedding to each edge token across all subsets:
\begin{equation}
\widetilde{H}_E[b, t, e, :, :] \leftarrow \widetilde{H}_E[b, t, e, :, :] + H_\rho[b, u, v, :]
\end{equation}
closing the local-global loop at each layer.

\subsubsection{Global Head}
\label{appendix:arch-global-head}
An auxiliary linear head $W_{\text{global}} \in \mathbb{R}^{d \times |\mathcal{C}|}$ is applied to $H_\rho$ to predict edge classes over all ordered node pairs:
\begin{equation}
\Theta^{\text{global}}_{b,i,j} = W_{\text{global}}\, H_\rho[b,i,j,:]
\end{equation}
supervised with cross-entropy over all positive edges and a sampled fraction of negatives. At inference, global logits can optionally be fused with marginal logits for non-candidate edges via a blend parameter $\alpha$:
\begin{equation}
\Theta^{\text{fused}}_{b,u,v} = \begin{cases} (1-\alpha)\,\hat{\Theta}^{\text{marginal}}_{b,u,v} + \alpha\,\Theta^{\text{global}}_{b,u,v} & \text{if } \{u,v\} \in K_{\text{cand}} \\ \Theta^{\text{global}}_{b,u,v} & \text{otherwise} \end{cases}
\end{equation}
where $K_\mathrm{cand}$ are the candidate edges.

\subsubsection{Optimization Objective}
\label{appendix:arch-opt}
The full training objective is:
\begin{equation}
\mathcal{L} = \mathcal{L}_{\text{CE}} + \lambda_{\text{SHD}}\,\mathcal{L}_{\text{SHD}} + \lambda_{\text{rank}}\,\mathcal{L}_{\text{rank}} + \lambda_{\text{global}}\,\mathcal{L}_{\text{global}} + \lambda_{\text{gate}}\,\mathcal{L}_{\text{gate}}
\end{equation}
$\mathcal{L}_{\text{CE}}$ is cross-entropy over candidate edges with temperature $\tau$ and optional class-frequency reweighting. $\mathcal{L}_{\text{SHD}}$ is a soft SHD penalty:
\begin{equation}
\mathcal{L}_{\text{SHD}} = \frac{1}{|K_{\text{cand}}|}\sum_{\{u,v\} \in K_{\text{cand}}} \left[(\hat{p}_{u\to v} - \mathbf{1}[u\to v])^2 + (\hat{p}_{v\to u} - \mathbf{1}[v\to u])^2\right]
\end{equation}
where $\hat{p}_{u \to v}$ corresponds to the predicted probability of edge $e = (u \to v)$ and $\mathcal{L}_{\text{rank}}$ is a pairwise hinge loss ranking true directed edges above non-edges with margin $\mu$:
\begin{equation}
\mathcal{L}_{\text{rank}} = \frac{1}{|\mathcal{P}|}\sum_{(e^+, e^-) \in \mathcal{P}} \max(0,\ \mu - \hat{p}_{e^+} + \hat{p}_{e^-})
\end{equation}
where $\mathcal{P}$ is a set of sampled positive/negative edge pairs. A gate regularization term encourages gates to remain active throughout training, where $\bar{g}_h$ is the mean gate value across edges and subsets:
\begin{equation}
\mathcal{L}_{\text{gate}} = \frac{1}{3}\sum_{h \in \{\text{cvt, ssi, dpt}\}} (1 - \bar{g}_h)^2, \qquad \bar{g}_h = \frac{1}{BTK}\sum_{b,t,e} g_{h,b,t,e}
\end{equation}
The model is optimized with AdamW with linear warmup and cosine learning rate decay to a minimum ratio $r_{\min}$, and EMA weights with decay $\beta$ are used at evaluation.

\section{Synthetic Data Generation}\label{appendix:datagen}

\paragraph{Graph generation} Synthetic datasets were generated by sampling a random DAG with a fixed number of nodes $N$ and an expected number of edges $E$. Graphs are generated by (1) sampling a random topological ordering of the nodes, (2) adding directed edges in this order with a fixed edge probability, (3) connecting disconnected components (if any) by adding a single edge consistent with the topological order. For interventional regime, a random subset of non-root nodes are sampled independently with probability $p_{interv}$ (default value 0.4) as intervention targets.

\paragraph{Dataset generation} Given a graph and intervention targets, we generate $S$ observations for each regime using regime-specific causal mechanisms.

\paragraph{Causal mechanisms.}
Each of the nodes $y$ with parents $X$ is assigned one of the following causal mechanisms to map its parent variables to its value, following the DCDI synthetic data generation setup. $W$ represents  randomly initialized weight matrices and $\varepsilon$ is an independent noise variable:
\begin{itemize}
    \item \textbf{Linear:} $y = XW + \varepsilon$
    \item \textbf{Polynomial:} $y = W_0 + XW_1 + X^2W_2+ \times \varepsilon$
    \item \textbf{Sigmoid:} $y=\sum^d_{i=1}W_i\cdot \operatorname{sigmoid}(X_i) + \times \varepsilon$
    \item \textbf{Randomly initialized neural network (NN):} $y = \operatorname{Tanh}((X,\varepsilon)W_{in})W_{out}$
    \item \textbf{Randomly initialized neural netowrk, additive (NN additive):} $y = \operatorname{Tanh}(XW_{in})W_{out} + \varepsilon$
\end{itemize}

Similarly, DCDI defaults were used for root causal mechanisms and noise variables:
\begin{itemize}
    \item Root node values were sampled from $\operatorname{Uniform(-2,2)}$.
    \item Noise variables were sampled independently for each node and regime as $\varepsilon \sim 0.4 \cdot \mathcal{N}(0, \sigma^2)$ where $\sigma^2 \sim \operatorname{Uniform}(1,2)$.
\end{itemize}

\paragraph{Soft interventions.}
To generate two contrastive regimes, we apply soft interventions to a random subset of non-root nodes.  For observational case, we only use the regime with no interventions.

Soft interventions are implemented by sampling the mechanism weights of intervened nodes from the same distribution with modified parameter range, while keeping the noise distribution the same. Non-intervened nodes retain their original weight distribution. As a result, the two regimes share the same underlying graph but conditional distributions differ.

\paragraph{Training and test datasets.}
Training datasets consist of 512 independently generated graphs with varying numbers of nodes and edges, and $S=10,000$. Validation and test datasets contain 10 graphs each.

In the out-of-distribution setting, training datasets are generated by uniformly sampling from Linear, NN, and NN Additive mechanisms per node with Polynomial and Sigmoid mechanisms held out and used only for validation and test datasets. For in-distribution experiments, we generate training, validation, and test datasets using Polynomial mechanism only.

Random seeds are fixed for each dataset to ensure reproducibility.

\section{Extended Results}\label{appendix:extresults}

\subsection{metrics}\label{appendix:metrics}

For benchmarking we report the following standard causal discovery metrics:

\paragraph{SHD} Structural Hamming Distance computes the number of edge additions, deletions, and reversals required to transform the predicted graph into the true graph, with lower values indicating better performance. We average SHD across all graphs in the validation dataset.

\paragraph{Precision, recall and F1} are computed for each edge type (no edge, edge) in the directed adjacency graph and are averaged across all graphs in the validation set.

\subsection{Benchmarking results}\label{appendix:fullbench}

We present full experimental results below with all metrics from Appendix ~\ref{appendix:metrics}. Note that we use publicly available implementations of all baselines with their default hyperparameters and recommended settings. We do not perform additional hyperparameter sweeps or per-dataset finetuning for AVICI and SEA, since both methods claim that no additional finetuning is needed. Public implementation of SEA included unresolved runtime errors that we fixed before running experiments. No hyperparameter sweeps were performed for SCONE.

\paragraph{In-distribution Benchmarking} In the in-distribution setting, all training, validation, and test graphs are generated using the Polynomial causal mechanism. Structural recovery is  evaluated on 10 graphs with 20 nodes and 20 edges. The results are averaged across graphs and evaluated using precision, recall, F1, and SHD (Figure ~\ref{fig:in-dist-appendix}, Table~\ref{tab:results_poly_indist}).

\begin{figure}[H]
    \centering
    \includegraphics[width=\linewidth]{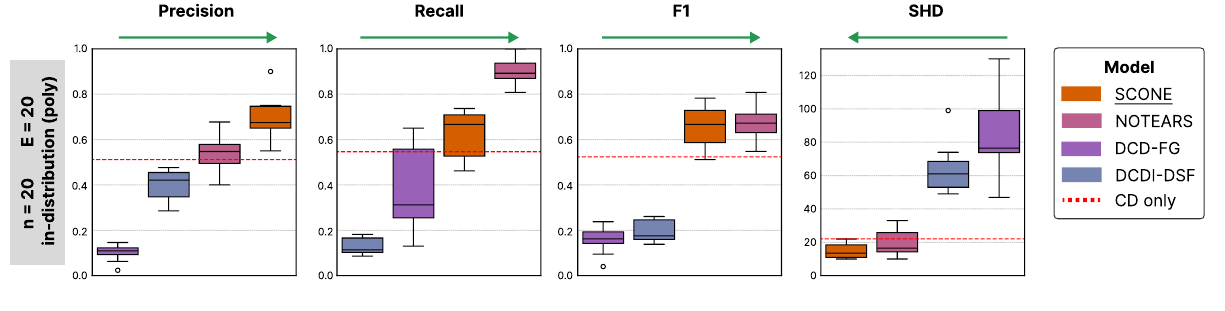}
    \caption{Precision, recall, F1 and SHD for state-of-the-art methods that only perform in-distribution graph prediction on 20 nodes 20 edges. 10 graphs were generated using the polynomial mechanism. The horizontal red dashed line indicates mean across 10 graphs on predictions from running Polynomial-BIC baseline only.}
    \label{fig:in-dist-appendix}
\end{figure}

\begin{table}[H]
\centering
\begin{tabular}{lcccc}
\toprule
Model & SHD $\downarrow$ & F1 $\uparrow$ & Precision $\uparrow$ & Recall $\uparrow$ \\
\midrule
\textsc{\textbf{Scone}}  & $\textbf{14.6} \pm 4.4$ & $0.655 \pm 0.094$ & $\textbf{0.694} \pm 0.098$ & $0.624 \pm 0.106$ \\
\textsc{Notears}         & $19.7 \pm 7.9$ & $\textbf{0.675} \pm 0.073$ & $0.544 \pm 0.081$ & $\textbf{0.899} \pm 0.058$ \\
\textsc{Dcdi-Dsf}        & $63.9 \pm 14.9$ & $0.195 \pm 0.048$ & $0.396 \pm 0.069$ & $0.130 \pm 0.036$ \\
\textsc{Dcd-Fg}          & $86.7 \pm 25.1$ & $0.157 \pm 0.057$ & $0.101 \pm 0.035$ & $0.377 \pm 0.190$ \\
\midrule \midrule
\textsc{Poly-BIC CD} & 22.1 & 0.524 & 0.511 & 0.546 \\
\bottomrule
\end{tabular}
\caption{Comparison of causal discovery methods on polynomial in-distribution graphs with $N=20$ nodes and $E=20$ edges. Results are averaged over 10 graphs (mean $\pm$ std).}
\label{tab:results_poly_indist}
\end{table}

\paragraph{Out-of-distribution generalization benchmarking}
We next evaluate generalization to unseen causal mechanisms. In this setting, to generate training dataset, node-level mechanisms are sampled from Linear, NN, and NN Additive families, while evaluation is performed on graphs generated using held-out mechanisms, Sigmoid and Polynomial. We additionally vary graph density (20 nodes with 20 and 30 edges) and size (50 nodes with 50 edges) to assess robustness under structural variation among methods capable of scaling to these settings. The results for each setting are given in Figure ~\ref{fig:ood-appendix} and Table~\ref{tab:results_main}.

\begin{figure}[H]
    \centering
    \includegraphics[width=\linewidth]{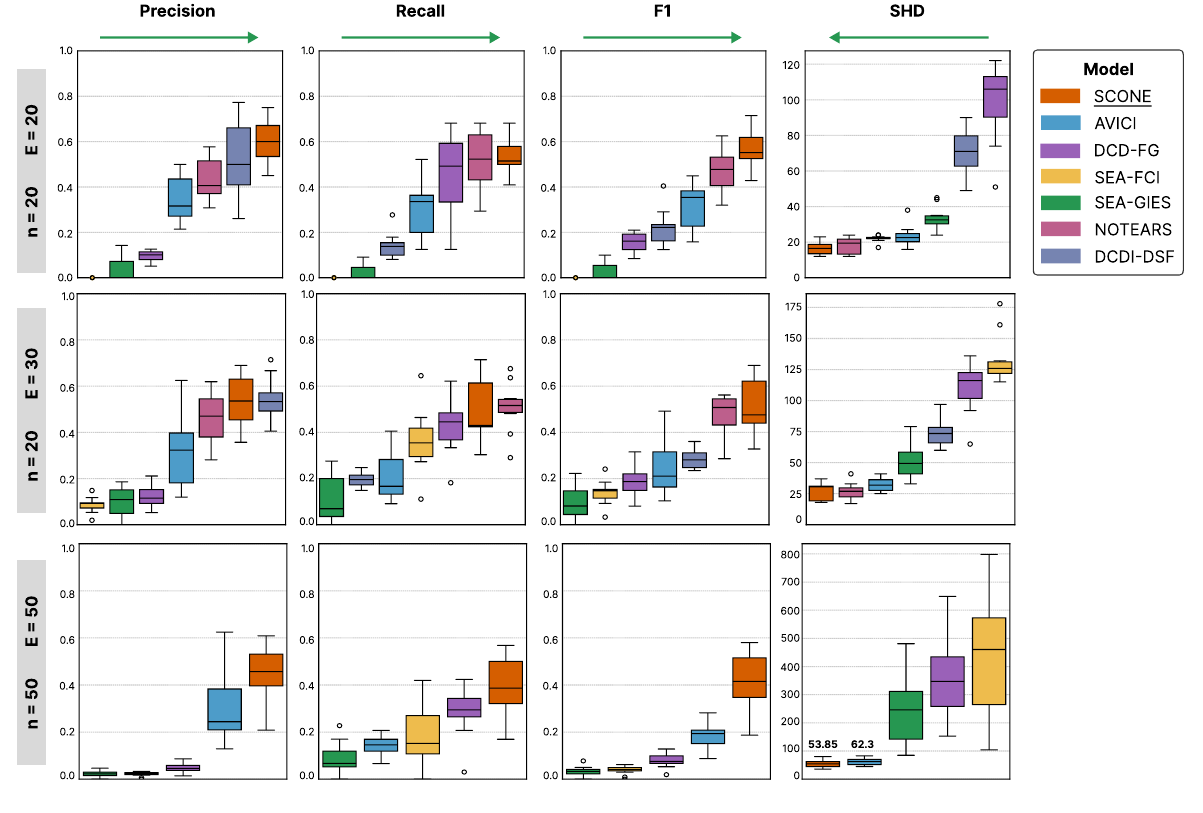}
    \caption{Precision, recall, F1 and SHD for state-of-the-art methods compared to SCONE on predictions over 10 graphs with 20 nodes 20 edges, 20 nodes 30 edges and 20 graphs for 50 nodes 50 edges. SEA-FCI obtains a precision, recall and F1 score of 0.0 in the 20 node 20 edge example, and is represented by a circle on the bottom of the axis.}
    \label{fig:ood-appendix}
\end{figure}

\begin{table}[H]
\centering
\begin{tabular}{lllcccc}
\toprule
$N$ & $E$ & Model & SHD $\downarrow$ & F1 $\uparrow$ & Precision $\uparrow$ & Recall $\uparrow$ \\
\midrule
& & \textsc{\textbf{Scone}} & $\textbf{16.3} \pm 3.7$ & $\textbf{0.571} \pm 0.098$ & $\textbf{0.610} \pm 0.110$ & $\textbf{0.537} \pm 0.091$ \\
\raisebox{1.5ex}{20} & \raisebox{1.5ex}{20} & \textsc{Avici}          & $23.3 \pm 6.2$ & $0.319 \pm 0.099$ & $0.346 \pm 0.098$ & $0.310 \pm 0.127$ \\
& & \textsc{Notears}        & $18.0 \pm 4.6$ & $0.468 \pm 0.105$ & $0.434 \pm 0.091$ & $0.517 \pm 0.139$ \\
& & \textsc{Dcd-Fg} & $98.8 \pm 22.2$ & $0.155 \pm 0.044$ & $0.096 \pm 0.024$ & $0.450 \pm 0.183$ \\
& & \textsc{Dcdi-Dsf}       & $70.3 \pm 13.5$ & $0.221 \pm 0.082$ & $0.512 \pm 0.181$ & $0.142 \pm 0.057$ \\
& & \textsc{Sea} (FCI)      & $21.9 \pm 2.0$ & $0.000 \pm 0.000$ & $0.000 \pm 0.000$ & $0.000 \pm 0.000$ \\
& & \textsc{Sea} (GIES)     & $33.5 \pm 6.7$ & $0.027 \pm 0.038$ & $0.038 \pm 0.055$ & $0.023 \pm 0.032$ \\
\midrule
\midrule
& & \textsc{\textbf{Scone}} & $\textbf{26.700} \pm 7.119$ & $\textbf{0.513} \pm 0.124$ & $0.536 \pm 0.114$ & $0.495 \pm 0.137$ \\
\raisebox{1.5ex}{20} & \raisebox{1.5ex}{30} & \textsc{Avici}          & $32.6 \pm 5.7$ & $0.243 \pm 0.119$ & $0.312 \pm 0.153$ & $0.205 \pm 0.102$ \\
& & \textsc{Notears}        & 26.900 $\pm$ 7.031 & 0.478 $\pm$ 0.090 & 0.462 $\pm$ 0.113 & $\textbf{0.509} \pm 0.110$  \\
& & \textsc{Dcd-Fg} & $110.9 \pm 20.9$ & $0.187 \pm 0.065$ & $0.121 \pm 0.047$ & $0.425 \pm 0.119$ \\
& & \textsc{Dcdi-Dsf}       & $74.3 \pm 10.985$ & $0.284 \pm 0.044$ & $\textbf{0.545} \pm 0.092$ & $0.194 \pm 0.031$ \\
& & \textsc{Sea} (FCI)      & $132.7 \pm 20.5$ & $0.139 \pm 0.055$ & $0.086 \pm 0.035$ & $0.364 \pm 0.139$ \\
& & \textsc{Sea} (GIES)     & $51.4 \pm 15.0$ & $0.094 \pm 0.077$ & $0.096 \pm 0.068$ & $0.109 \pm 0.102$ \\
\midrule
\midrule
& & \textsc{\textbf{Scone}} & $\textbf{53.85} \pm 11.60$ & $\textbf{0.424} \pm 0.102$ & $\textbf{0.455} \pm 0.102$ & $\textbf{0.398} \pm 0.103$ \\
\raisebox{1.5ex}{50} & \raisebox{1.5ex}{50} & \textsc{Avici}          & $62.3 \pm 11.9$ & $0.188 \pm 0.050$ & $0.299 \pm 0.128$ & $0.142 \pm 0.036$ \\
& & \textsc{Dcd-Fg} & $373.95 \pm 149.6$ & $0.082 \pm 0.027$ & $0.049 \pm 0.019$ & $0.294 \pm 0.087$ \\
& & \textsc{Sea} (FCI)      & $439.25 \pm 198.0$ & $0.039 \pm 0.015$ & $0.023 \pm 0.008$ & $0.184 \pm 0.114$ \\
& & \textsc{Sea} (GIES)     & $242.10 \pm 109.3$ & $0.033 \pm 0.016$ & $0.022 \pm 0.010$ & $0.084 \pm 0.060$ \\
\bottomrule
\end{tabular}
\caption{Comparison of causal discovery methods across graph sizes. Results are averaged over 10 graphs for $N=20$ and 20 graphs for $N=50$ (mean $\pm$ std).}
\label{tab:results_main}
\end{table}

\paragraph{Scalability to larger graphs}
We evaluate scalability on graphs with 100 nodes and 100 edges. Results, averaged over 20 graphs, are reported in Table ~\ref{tab:resultsscale-appendix}. SCONE achieves the lowest SHD and markedly higher F1 and precision compared to all baselines. In contrast, DCD-FG and SEA (FCI) predict dense graphs, leading to inflated recall at the cost of low precision and large SHD.

\begin{table}[H]
\centering
\begin{tabular}{lcccc}
\toprule
Model & SHD $\downarrow$ & F1 $\uparrow$ & Precision $\uparrow$ & Recall $\uparrow$ \\
\midrule
\textsc{\textbf{Scone}}      & $\textbf{126.7} \pm 25.0$ & $\textbf{0.237} \pm 0.046$ & $\textbf{0.345} \pm 0.074$ & $0.186 \pm 0.046$ \\
\textsc{Avici}      & $131.3 \pm 18.0$ & $0.100 \pm 0.042$ & $0.210 \pm 0.082$ & $0.068 \pm 0.031$ \\
\textsc{Dcd-Fg}     & $1001.3 \pm 537.9$ & $0.052 \pm 0.011$ & $0.030 \pm 0.008$ & $0.230 \pm 0.076$ \\
\textsc{Sea} (FCI)  & $3178.8 \pm 425.2$ & $0.022 \pm 0.004$ & $0.012 \pm 0.002$ & $\textbf{0.325} \pm 0.056$ \\
\textsc{Sea} (GIES) & $1797.2 \pm 296.6$ & $0.022 \pm 0.004$ & $0.012 \pm 0.002$ & $0.185 \pm 0.044$ \\
\bottomrule
\end{tabular}
\caption{Comparison of SHD, F1, precision and recall for causal discovery methods on graphs with $N=100$ nodes and $E=100$ edges. Results are averaged over 20 graphs (mean $\pm$ std).}
\label{tab:resultsscale-appendix}
\end{table}

\paragraph{Ablation on Bias Heads} We evaluate the contribution of the contrastive bias heads (SSI, CVT, DPT) by comparing \textsc{SCONE} to \textsc{SCONE-NB} with bias heads removed. Table ~\ref{tab:resultsabl-sconenbapp} reports results for 20-node 40-edge and 50-node 50-edge graphs, averaged over 20 graphs per setting.

\begin{table}[H]
\centering
\begin{tabular}{lllcccc}
\toprule
$N$ & $E$ & Model & SHD $\downarrow$ & F1 $\uparrow$ & Precision $\uparrow$ & Recall $\uparrow$ \\
\midrule
& & \textsc{\textbf{Scone}} & $\textbf{39.0} \pm 9.7$ & $\textbf{0.469} \pm 0.143$ & $\textbf{0.473} \pm 0.138$ & $\textbf{0.466} \pm 0.153$ \\
\raisebox{1.5ex}{20} & \raisebox{1.5ex}{40} & \textsc{Scone-NB} & $39.6 \pm 9.0$ & $0.453 \pm 0.123$ & $0.460 \pm 0.126$ & $0.449 \pm 0.125$ \\
\midrule
& & \textsc{\textbf{Scone}} & $\textbf{53.85} \pm 11.60$ & $\textbf{0.424} \pm 0.102$ & $\textbf{0.455} \pm 0.102$ & $\textbf{0.398} \pm 0.103$ \\
\raisebox{1.5ex}{50} & \raisebox{1.5ex}{50} & \textsc{Scone-NB} & $56.60 \pm 10.79$ & $0.393 \pm 0.092$ & $0.424 \pm 0.094$ & $0.368 \pm 0.092$ \\
\bottomrule
\end{tabular}
\caption{Comparison of SCONE with and without orientation rules. Results for $N=20$ are averaged over 10 graphs and $N=50$ over 20 graphs (mean $\pm$ std). \textsc{Scone-NB} ablates the contrastive bias heads.}
\label{tab:resultsabl-sconenbapp}
\end{table}

\paragraph{Ablation Contrastive} Contrastive features such as reparametrization, contrastive sampling, edge embedding pairwise statistics, were ablated (\textsc{SCONE-NC}) to empirically validate the importance of invariance/contrast for recovering additional edges. Results for 20-node 40-edge and 50-node 50-edge graphs were averaged over 20 graphs and are presented in Table~\ref{tab:resultsabl-sconencapp}.

\begin{table}[H]
\centering
\begin{tabular}{lllcccc}
\toprule
$N$ & $E$ & Model & SHD $\downarrow$ & F1 $\uparrow$ & Precision $\uparrow$ & Recall $\uparrow$ \\
\midrule
& & \textsc{\textbf{Scone}} & $\textbf{39.0} \pm 9.7$ & $\textbf{0.469} \pm 0.143$ & $\textbf{0.473} \pm 0.138$ & $\textbf{0.466} \pm 0.153$ \\
\raisebox{1.5ex}{20} & \raisebox{1.5ex}{40} & \textsc{Scone-NC} & $48.6 \pm 9.3$ & $0.341 \pm 0.136$ & $0.350 \pm 0.136$ & $0.335 \pm 0.139$ \\
\midrule
& & \textsc{\textbf{Scone}} & $\textbf{53.85} \pm 11.60$ & $\textbf{0.424} \pm 0.102$ & $\textbf{0.455} \pm 0.102$ & $\textbf{0.398} \pm 0.103$ \\
\raisebox{1.5ex}{50} & \raisebox{1.5ex}{50} & \textsc{Scone-NC} & $66.55 \pm 9.1$ & $0.322 \pm 0.076$ & $0.350 \pm 0.076$ & $0.300 \pm 0.077$ \\
\bottomrule
\end{tabular}
\caption{Comparison of SCONE with and without any contrastive features. Results for $N=20$ are averaged over 10 graphs and $N=50$ over 20 graphs (mean $\pm$ std). \textsc{Scone-NC} ablates the contrastive features in the model.}
\label{tab:resultsabl-sconencapp}
\end{table}

\end{document}